\documentclass[twoside,11pt]{article}

\PassOptionsToPackage{hyperfootnotes=false}{hyperref}
\usepackage{jmlr2e}

\usepackage{lastpage}
\jmlrheading{}{2026}{1-\pageref{LastPage}}{\today}{}{}{Binh Nguyen, Trinh Tran, Hung D. Nguyen and Truong X. Nghiem}

\ShortHeadings{LEAF: A Learning-Enabled ADMM Framework}{Nguyen, Tran, Nguyen, and Nghiem}
\firstpageno{1}

\usepackage{amsmath,amssymb,amsfonts}
\allowdisplaybreaks
\usepackage{algorithm}
\usepackage{algpseudocode}
\usepackage{graphicx}
\usepackage{textcomp}
\usepackage[dvipsnames]{xcolor}
\usepackage{soul}
\usepackage{tabularray}
\usepackage{graphicx}
\usepackage{array}
\usepackage{subcaption}
\usepackage{adjustbox}
\usepackage[utf8]{inputenc}
\usepackage{mathtools}
\usepackage{balance}
\usepackage{hyperref}
\makeatletter
\@ifpackageloaded{jmlr2e}{}{\usepackage{authblk}}
\makeatother
\usepackage{multirow}
\usepackage{hyperref}
\usepackage{stfloats}
\usepackage{siunitx}
\usepackage{comment}

\usepackage[opt, matmatrix]{icpslab}

\renewcommand{\subjectto}{\ensuremath{\text{s.t.}}}

\definecolor{lightskyblue}{RGB}{135, 206, 250}
\definecolor{revorange}{RGB}{240, 120, 0}
\newcommand{\rev}[1]{\textcolor{revorange}{#1}}

\makeatletter
\@ifpackageloaded{jmlr2e}{}{
\newtheorem{remark}{Remark}
\newtheorem{lemma}{Lemma}
\newtheorem{definition}{Definition}
\newtheorem{theorem}{Theorem}
\newtheorem{proposition}{Proposition}
}
\makeatother

\DeclareMathOperator{\prox}{\mathbf{prox}}
\DeclareMathOperator{\Dom}{Dom}
\newcommand{\ICNN}{\mathrm{ICNN}}

\def\bfM{\mathbf{M}}

\def\bfx{\mathbf{x}}
\def\bfu{\mathbf{u}}

\begin{document}

\title{LEAF: A Learning-Enabled ADMM Framework for Accelerated Convex Optimization}

\author{\name Binh Nguyen \textsuperscript{1} \email thanhbinh.nguyen@ucf.edu \\
\name Trinh Tran \textsuperscript{1} \email trinh.tran@ucf.edu \\
\name Hung D. Nguyen \textsuperscript{2} \email hunghtd@ntu.edu.sg \\
\name Truong X. Nghiem\thanks{Corresponding author.}\hspace{6pt}\textsuperscript{1} \email truong.nghiem@ucf.edu \\
\addr $^1$ Department of Electrical and Computer Engineering,
College of Engineering and Computer Science\\ University of Central Florida, Orlando, FL 32816, USA\\
$^2$ School of EEE, Nanyang Technological University, Singapore. 
}

\editor{}

\maketitle

\begin{abstract}
We propose LEAF, a machine-learning-accelerated ADMM framework for constrained convex optimization.
LEAF uses an input-convex neural network (ICNN) to approximate the scalara-valued Moreau envelope of the objective function and evaluates its gradient to accelerate a costly primal update within ADMM.
The learned component reduces per-iteration computation while retaining ADMM's iterative structure.
This leads to the proposed Moreau Envelope Learning ADMM (MEL-ADMM) and its splitting variant sMEL-ADMM.
The ICNN architecture encodes convexity, and smooth activations provide the regularity required for the learned model.
Both MEL-ADMM and sMEL-ADMM are developed with theoretical guarantees of convergence and feasibility under the learned model.
Rigorous analysis shows that the proposed methods achieve convergence rates comparable to classical ADMM while reducing per-iteration computational cost.
Numerical experiments demonstrate up to an order-of-magnitude speedup over state-of-the-art solvers while maintaining low optimality gaps.
\end{abstract}


\section{Introduction}
\label{sec:intro}

Over the past decades, the Alternating Direction Method of Multipliers (ADMM) has attracted significant research attention and has undergone substantial theoretical and algorithmic developments \citep{ariafar2019admmbo,nguyen2022collisionfree, Duarte2026communicationefficient}.
Typically, ADMM is useful when an objective consists of multiple components that depend on different variables,  while those variables must satisfy local and possibly coupling constraints.
Due to its broad capability, ADMM has been ubiquitous in many applications, including  image processing \citep{barber2024convergence},  power grid systems \citep{hasanzadeh2025admm}, communication networks \citep{elgabli2020gadmm}, neural network training \citep{zeng2021admm}, and cooperative planning of multiple robots \citep{nguyen2024connectivity}.
Despite these advantages, ADMM can be computationally expensive when a subproblem, such as the proximal step of a complex objective, has no closed-form solution and must be solved numerically at every iteration.
This cost is incurred at each iteration, whether one solves a single problem or a sequence of related instances.

Machine learning offers a way to approximate expensive computations and, when multiple instances are available, exploit information shared across them.
Learning-assisted solvers use data to predict initial points or active sets, or to approximate expensive operators while retaining the iterations of the original algorithm \citep{baker2019learning,misra2022learning,meinhardt2017learning}.
Learning-to-optimize (L2O) methods, in contrast, often predict a solution directly or learn the update rule itself \citep{chen2022learning,sacks2022learning,sucker2025learning}.
These methods can achieve low computational cost, but their solution quality and constraint satisfaction may depend on the training distribution and on additional feasibility mechanisms.
When the original solver has useful structure and guarantees, learning a specific computational bottleneck offers an advantage that the learned component reduces computation while the solver continues to handle the optimization problem.
Existing learning-assisted solvers of this kind typically approximate the bottleneck computation with a generic, black-box model, trained only for accuracy, so the guarantees of the original solver are generally lost once the learned component is inserted.
In contrast, the learned model should be informed by the structure and properties of the computation it replaces, not only to approximate it accurately but also to retain the guarantees of the solver.
This motivates our use of machine learning to \emph{accelerate ADMM without replacing its iterative structure}, through a \emph{learned model whose structure is chosen so that the convergence analysis of ADMM still applies}.

This paper proposes LEAF (Learning-Enabled ADMM Framework), a machine-learning-accelerated ADMM method for constrained convex optimization.
LEAF retains the ADMM iterations and learns a surrogate for the Moreau envelope (ME) used in a costly primal update.
The learned component is designed to respect the convex structure of the problem, allowing computational savings within an ADMM framework with convergence and feasibility guarantees.
In particular, this work develops a supervised learning model based on Input Convex Neural Networks (ICNNs) to approximate the ME of the convex objective function, yielding a scalar output and ensuring the smoothness and convexity of the ME.
The model is then embedded in ADMM to accelerate its computation while preserving the convergence and feasibility guarantees.
This scalar-valued model avoids directly learning a vector-valued solution map or proximal operator, while the ICNN architecture encodes convexity in the approximation.
LEAF supports a wide range of optimization problems, including non-smooth formulations with dynamic constraints.
Based on LEAF, we propose two algorithms, Moreau-Envelope Learning ADMM (MEL-ADMM) and splitting MEL-ADMM (sMEL-ADMM), that ensure global convergence and feasibility.
Rigorous theoretical analysis shows that the convergence rate of the proposed algorithms is comparable to ADMM, while achieving a lower per-iteration computational cost.
In several applications, our methods outperform existing methods, achieving up to ten times faster solving time within acceptable optimality gaps, even with relatively compact learning models.

Our main contributions are summarized as follows:
\begin{itemize}
\item \textit{An effective learning-enabled ADMM framework (LEAF)}: We propose a general framework that integrates machine learning into ADMM to bypass expensive primal updates through efficient forward computations.
\item \textit{Convexity informed learning}: We employ ICNNs to approximate Moreau envelopes while explicitly enforcing convexity and smoothness, leading to highly accurate and data-efficient learning.
\item \textit{Theoretical guarantees:} We propose two algorithms with rigorous theoretical analysis ensuring global convergence and feasibility: MEL-ADMM iterates are feasible at every iteration, while sMEL-ADMM satisfies the equality constraints at every iteration and the remaining constraints in the limit.
\item \textit{Comprehensive experimental validation}: Numerical experiments on large-scale convex optimization problems demonstrate significant computational speedups over state-of-the-art solvers.
\end{itemize}

The remainder of this paper is organized as follows.
Section \ref{sec:lit-review} reviews related work.
Section \ref{sec:problem-formulation} presents the problem formulation and preliminaries.
Section \ref{sec:framework} introduces LEAF and the associated algorithms.
Section \ref{sec:analysis} provides theoretical analysis and convergence guarantees.
Section \ref{sec:MPC} presents an MPC application of LEAF.
Section \ref{sec:experiments} reports numerical experiments.
Finally, Section \ref{sec:conclusion} concludes the paper.

\section{Literature review}
\label{sec:lit-review}

This section reviews machine learning for optimization, distinguishing learning-assisted iterative solvers from learning-to-optimize (L2O) methods according to what is learned and how it is used.
We focus on three aspects: (i) learning approaches, (ii) modeling strategies and structural priors, and (iii) feasibility and optimality considerations.

\noindent {\bf Learning Approaches.}
The integration of machine learning into optimization to improve computational efficiency and scalability has been extensively studied \citep{chen2022learning,bengio2021machine}.
The related approaches considered here fall into three broad classes.

\begin{itemize}
\item The first class consists of end-to-end learning methods that directly map problem data to candidate solutions.
These methods typically rely on supervised learning and require a large number of training samples generated by conventional solvers \citep{koziel2013surrogate, zamzam2020learning}.
To improve solution quality, several works incorporate optimality conditions, such as the Karush–Kuhn–Tucker (KKT) conditions, into the training process \citep{Saket2021CDC,peiris2025kkt}.
Constraint satisfaction is often addressed through penalty-based losses or feasibility-seeking mechanisms \citep{dontiDC3Learning2021, nguyen2025fsnet}.
However, these methods generally lack strict guarantees of feasibility and optimality, and their performance is sensitive to the coverage of the training data distribution.
As a result, they may exhibit degraded performance when encountering out-of-distribution problem instances.
Moreover, these methods generally depend heavily on the coverage of the training dataset, which must include representative problem instances and their corresponding true optimal solutions as labels.
In high-dimensional settings, such coverage is often limited, which can reduce the accuracy and generalization performance of end-to-end learning models.

\item The second class retains an optimization solver and uses learned components to accelerate its iterations.
Representative works include learning warm-start initializations \citep{baker2019learning, sambharya2024learning} and predicting active constraint sets to facilitate constraint elimination \citep{misra2022learning}.
Such methods reduce total computational cost by decreasing the iteration count or the cost of individual updates while retaining the solver structure.
Another line of research focuses on approximating computationally expensive operators within each iteration using data-driven models \citep{Chang2017OneNetwork, meinhardt2017learning}.
While effective in accelerating individual steps, these methods typically involve learning high-dimensional mappings whose outputs match the dimension of the decision variables, which can limit scalability and weaken structural guarantees.
LEAF belongs to this solver-acceleration line of work: it retains ADMM but learns the Moreau envelope instead of directly approximating the proximal operator.
The learned scalar-valued function supplies a cheaper primal update through its gradient while preserving the convex structure used in the analysis.

\item The third class designs end-to-end learning architectures by unrolling the iterations of classical optimization algorithms into trainable networks.
Representative examples include ADMM-based deep learning models \citep{sun2016deep, Yang2020ADMMCSNet} and learned optimizers based on gradient descent dynamics \citep{andrychowicz2016learning, chen2017learning}.
While these methods can achieve efficient inference by learning optimization dynamics, they often lack rigorous convergence guarantees and typically require a substantial number of training samples.
\end{itemize}

\noindent {\bf Modeling approaches and structural priors.}
Methods that predict solutions or approximate proximal operators directly learn vector-valued maps \citep{zamzam2020learning,meinhardt2017learning}.
LEAF instead uses supervised learning and input-convex neural networks (ICNNs) to approximate the scalar-valued Moreau envelope of the objective function.
ICNNs can enforce convexity of a scalar output \citep{amos2017input}, and the gradient of the Moreau envelope determines the corresponding proximal update \citep{parikh2014proximal}.
We use this relationship within ADMM, together with smooth activations that support the regularity conditions in our convergence analysis.

\noindent {\bf Feasibility and optimality.}
While end-to-end L2O approaches can deliver fast solutions, they often do not guarantee satisfaction of problem constraints.
Several constraint-enforcement approaches have been proposed to mitigate this issue, including penalty methods \citep{dontiDC3Learning2021, park2023self} and projection-based methods \citep{min2024hardnet}.
The former adds penalty terms to the loss function to penalize constraint violations during training; however, it lacks feasibility guarantees.
The latter projects the predicted solutions onto the feasible set to provide feasible solutions at the expense of larger optimality gaps.
For LEAF, convergence and feasibility guarantees follow under the conditions imposed on the learned model and the retained ADMM updates, as analyzed in Section~\ref{sec:analysis}; furthermore, the numerical examples report both low constraint violations and optimality gap.

\section{Problem Formulation}
\label{sec:problem-formulation}

Before detailing the problem formulation, we introduce the notations and important definitions used in this paper.
Let $\bbN_{>0}$ denote the set of positive integers.
For a function $f$, $\Dom(f)$ denotes its domain.
Given an ordered index set $\calI = \{i_1, i_2,\dots, i_n \} \subset \bbN_{>0}$ and real matrices $M_{i_1}, M_{i_2}, \dots, M_{i_n}$ in appropriate dimensions, let us denote their concatenation
\begin{align*}
[M_{i}]_{i \in \calI} = \begin{bmatrix}  M_{i_1}^\top &M_{i_2}^\top &\dots &M_{i_n}^\top \end{bmatrix}^\top.
\end{align*}
A function $f: X \rightarrow \bbR \cup \{\pm\infty\}$ is proper if $\Dom(f) \neq \emptyset$ and $f(x) > -\infty$ for all $x \in X$, is lower semi-continuous (l.s.c) at $x$ if $x_k \to x$ implies
$ \lim_{k \to \infty} f(x_k) \geq f(x)$, and
is l.s.c. on $X$ if $f$ is l.s.c. at every $x \in X$.
A function $f$ is convex if
$\forall x,y \in \Dom(f)$ and $\alpha \in [0, 1]$,
$\alpha x + (1-\alpha)y \in \Dom(f)$ and
$f(\alpha x + (1-\alpha)y) \leq \alpha f(x) + (1-\alpha) f(y)$.
Let $\Gamma_0(X)$ denote the set of all proper, lower semi-continuous, and convex functions on $X$ \citep{rockafellar2015convex}.
The indicator function of a convex set $\calC$ is defined as $I_\calC(x) = +\infty$ if $x \not\in \calC$ and $I_\calC(x) = 0$ if $x\in \calC$.
The distance from a point $x$ to a non-empty set $\calC \subseteq \bbR^n$ is denoted by $\operatorname{dist}(x,\calC) = \inf_{y\in\calC}\Vert x - y\Vert_2$.
A function $f$ is $L$-Lipschitz continuous with Lipschitz constant $L$ if
$\Vert f(x) - f(y) \Vert_2 \leq L \Vert x - y \Vert_2$ for all $x,y \in \Dom(f)$.

\subsection{Optimization problem}
\label{sec:general_opti}

This paper considers the following convex optimization problem
\begin{subequations}
\label{eq:general_opti}
\begin{optimization}{f(z)}
&A z = b, \label{eq:gen_opti_eq_cstr} \\
&z \in \calZ, \label{eq:gen_opti_ieq_str}
\end{optimization}
\end{subequations}
where $ z \in \bbR^n$ is the decision vector.
We assume that the objective function $f(z): \bbR^n \rightarrow \bbR \cup \{\pm\infty\}$ is convex, proper, and lower semi-continuous, that is, $f\in \Gamma_0(\bbR^n)$.
The matrix $A \in \bbR^{m\times n}$, vector $b \in \bbR^m$, and non-empty closed convex set $\calZ$ are problem parameters and may vary across problem instances.
We assume that the parameters $(A, b, \calZ)$ are such that the optimization problem \eqref{eq:general_opti} is feasible.

By introducing auxiliary variable $w \in \bbR^n$, the problem \eqref{eq:general_opti} is equivalent to
\begin{subequations}
\label{eq:general_opti_con}
\begin{optimization}{f(z) + I_{\bar\calZ}(w)}
&w - z = 0,
\end{optimization}
\end{subequations}
where $\bar{\calZ} = \{z \in \calZ\,|\,A z = b \}$.
The augmented Lagrangian of \eqref{eq:general_opti_con} with the penalty parameter $\rho > 0$ is given by
\begin{equation}
\calL(z, w, \alpha) = f(z) \!+\!
I_{\bar{\calZ}} (w) \!+\! \alpha^\top(w \!-\! z)
\!+\! \frac{\rho}{2} \Vert w \!-\! z \Vert^2_2\text.
\label{eq:lagrangian}
\end{equation}
Applying the over-relaxed ADMM iterations \citep{boydDistributedOptimization2011} to the augmented Lagrangian~\eqref{eq:lagrangian} yields
\begin{subequations}
\label{eq:ADMM-gen}
\begin{align}
z^{i+1} &= \argmin f(z)  + \frac{\rho}{2} \Vert w^i - z +  \rho^{-1} \alpha^i\Vert_2^2,
\label{eq:ADMM-gen-z}
\\
\tilde z^{i+1} &=  \gamma z^{i+1} + (1-\gamma)w^i,
\\
w^{i+1} &= \Pi_{\bar{\calZ}}( \tilde z^{i+1} -  \rho^{-1} \alpha^i),
\label{eq:ADMM-gen-w}
\\
\alpha^{i+1} &= \alpha^i + \rho (w^{i+1} - \tilde z^{i+1}),
\label{eq:ADMM-gen-b}
\end{align}
\end{subequations}
where
$\gamma \in (0,2)$ is the relaxation parameter
and $\Pi_{\bar{\calZ}}$ denotes the Euclidean projection onto $\bar\calZ$.
We tune
$\gamma$ and $\rho$, specifically we choose $\gamma \in [1.0,1.8]$, which has been empirically shown to
improve the convergence rate \citep{nishihara2015general, eckstein1998operator}.

The computational bottleneck of the ADMM iteration \eqref{eq:ADMM-gen} is primarily associated with the subproblems \eqref{eq:ADMM-gen-z} and \eqref{eq:ADMM-gen-w}.
This burden is exacerbated when $f$ is non-quadratic or non-smooth, in which case solving \eqref{eq:ADMM-gen-z} becomes particularly expensive.

Addressing the above challenge, our \textbf{objective} is to efficiently compute \emph{feasible and near-optimal solutions} to \eqref{eq:general_opti} by leveraging machine learning to accelerate ADMM iterations while preserving convergence and feasibility guarantees.

\subsection{Moreau envelopes}

We recall the concept of Moreau envelopes.
Let $f: \bbR^n \rightarrow \bbR \cup \{ +\infty \}$ be a closed proper convex function.
The proximal operator $\prox_{\lambda f}: \bbR^n \rightarrow \bbR^n$ of the scaled function $\lambda f$, for $\lambda > 0$, is given by
\begin{align}
\label{eq:proximal_operator}
\prox_{\lambda f}(x) = \argmin_z \left\{ f(z) + \frac{1}{2\lambda} \Vert x - z \Vert_2^2 \right\}.
\end{align}
The Moreau envelope $M_{\lambda}f$ of $f$ is then defined as \citep{bauschke2020correction}
\begin{equation}
\label{eq:moreau}
M_{\lambda}f(x) = \min_z \left\{ f(z) + \frac{1}{2 \lambda }\|x -z\|_2^2\right\}\text.
\end{equation}
\begin{proposition}
\label{pro:pros}
\citep{bauschke2020correction, bertsekas-cvx-algos}
If $f \in \Gamma_0(\bbR^n)$ then
\begin{itemize}
\item[(i)] $M_{\lambda}f(x)$ is convex and $\nabla M_{\lambda}f (x)$ is Lipschitz continuous with constant $\frac{1}{\lambda}$,
\item[(ii)] $M_{\lambda}f(x) \leq f(x)$ for all  $x \in \bbR^n$.
\end{itemize}
\end{proposition}
Furthermore, a useful relationship exists between $M_{\lambda}f(x)$ and $\prox_{\lambda f}(x)$, expressed by \citet[Proposition~5.1.7]{bertsekas-cvx-algos}:
\begin{equation}
\label{eq:MErelationship}
\prox_{\lambda f}(x) = x - \lambda  \nabla M_{\lambda}f(x)\text.
\end{equation}
With the help of the Moreau envelope, substituting \eqref{eq:MErelationship} into \eqref{eq:ADMM-gen-z} with $\lambda = \frac{1}{\rho}$, the $z-$update \eqref{eq:ADMM-gen-z} turns into
\begin{align} \label{eq:ADMM-PCA-z-ME}
z^{i+1} = q^i - \frac{1}{\rho}\nabla M_{\frac{1}{\rho}}f (q^i),
\end{align}
where $q^i =  w^i + \rho^{-1} \alpha^i$.

\section{Learning-Enabled ADMM Framework (LEAF)}
\label{sec:framework}

To address the computational bottleneck in ADMM iterations, particularly the expensive proximal update in \eqref{eq:ADMM-gen-z}, we propose a \emph{learning-enabled ADMM framework} (LEAF) that replaces this step with a learned approximation.
The key idea is to approximate the Moreau envelope (ME) of the objective function using a differentiable learning model, enabling the proximal update to be computed efficiently via gradient evaluation.
Rather than learning to directly approximate the proximal operator by a neural network as in \citet{meinhardt2017learning}, our approach learns a scalar-valued representation of the ME $M_{\frac{1}{\rho}} f$, from which the required update can be obtained through its gradient by \eqref{eq:ADMM-PCA-z-ME}.
This formulation significantly reduces the learning complexity while preserving important structural properties.
To ensure convexity and theoretical consistency, we employ input convex neural networks (ICNNs) to model the ME.
Building on this idea, we develop the Moreau Envelope Learning (MEL) approach and integrate it into ADMM to obtain two algorithms, termed MEL-ADMM and its splitting variant sMEL-ADMM, to be presented in this section.

\begin{figure}
\centering
\includegraphics[width=0.7\linewidth]{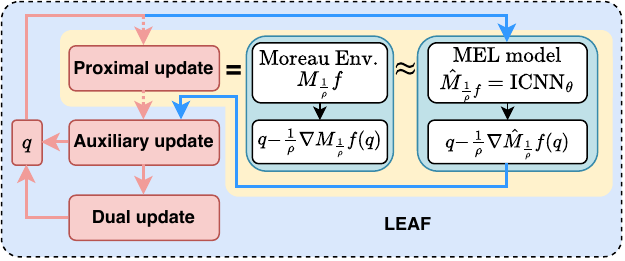}
\caption{Schematic overview of the Learning-Enabled ADMM Framework (LEAF).}
\label{fig:MEL}
\vspace{-1em}
\end{figure}

\subsection{Moreau envelope learning (MEL) model}
\label{sec:MEL-model}

To efficiently approximate the Moreau envelope $M_{\frac{1}{\rho}}f$ and its gradient $\nabla M_{\frac{1}{\rho}}f$, we develop a neural network-based model $\hat{M}_{\frac{1}{\rho}}f$ of $M_{\frac{1}{\rho}}f$ tailored for integration within ADMM iterations.
Previous studies that use Gaussian process models \citep{Duarte2026communicationefficient, nghiem2018learning} suffer from increasing inference cost as the dataset grows and lack theoretical guarantees, making them unsuitable for fast iterative optimization.
Neural networks enable fast inference and gradient evaluations, making them well-suited for replacing computationally expensive primal updates.

Since the ME is convex, the learning model should preserve this property to ensure consistency with the underlying optimization problem.
To enforce convexity, we employ an ICNN \citep{amos2017input}, which guarantees that the learned function is convex with respect to its input.
The structure of an $\ell$-layer ICNN with input $x$ and output $\ICNN_\theta(x)$ is
\begin{align}
\label{eq:ICNN}
\hat{M}_{\frac{1}{\rho}}f(x):
\begin{cases}
v_0 = \phi_0(W_0 x + b_0),\\
v_{j+1} = \phi_j (W_j v_j + V_j x + b_j),~j = 0, \dots, \ell-1,\\
\ICNN_\theta(x) = \phi_\ell(W_\ell v_\ell + V_\ell x + b_\ell),
\end{cases}
\end{align}
where $v_0$ represents the activation of the initial hidden layer, $W_0$ is the weight matrix connecting the input $x$ to this first layer, $\phi_j$ is the nonlinear activation function for layer $j$,
$\theta = (\{W_j, b_j\}_{j = 0,\dots,\ell}, \{V_j\}_{j=1, \dots, \ell})$ denotes the set of all parameters.
All functions $\phi_j$ are convex and non-decreasing, and all $W_j$ are non-negative.
Under these conditions, $\ICNN_\theta(x)$ is a convex function of $x$.
This architecture ensures convexity by constraining weights and activation functions, allowing the MEL model $\hat{M}_{\frac{1}{\rho}}f$ to inherit key properties of the ME.
Unlike direct proximal operator learning \citep{meinhardt2017learning}, this approach leverages ICNNs to learn the scalar-valued ME, reducing output dimensionality and improving scalability.
The overall framework is illustrated in Fig.~\ref{fig:MEL}, where the MEL model $\hat M_{\frac{1}{\rho}} f$ serves as an approximation of the ME $M_{\frac{1}{\rho}} f$, from which the proximal update is computed.

If the objective function $f$ is strongly convex, its ME $M_{\frac{1}{\rho}}f$ also inherits strong convexity, which can be exploited to further improve the learning model.
\begin{definition}[Strong Convexity]\label{def:strong_convexity}
A function $f \in \Gamma_0(\mathbb{R}^{n})$ is strongly convex with modulus $\sigma$ ($\sigma$-strongly convex) if and only if there exists a constant
$\sigma > 0$ such that  $f - \tfrac{\sigma}{2}\|\cdot\|^{2}$ is convex.
That is for all  $\lambda \in (0,1)$ and for all $x,y \in \mathbb{R}^{n}$,
\[
f(\lambda x + (1-\lambda)y)
\le \lambda f(x) + (1-\lambda)f(y)
- \tfrac{\sigma}{2}\lambda(1-\lambda)\|x-y\|^{2}\text.
\]
\end{definition}

\begin{lemma}{\citep[Lemma 2.23]{planidenStronglyConvex2016}}
\label{lem:strong-ME}
The function $f$ is strongly convex if and only if its Moreau envelope is strongly convex.
\end{lemma}

Therefore, when $f$ is strongly convex, the above result suggests a stronger structural prior for learning the ME by combining the $\ICNN_\theta$ with a quadratic function.
Finally, to ensure Lipschitz continuity of the gradient $\nabla M_{\frac{1}{\rho}}f$, we select activation functions that are non-decreasing, twice differentiable, and have bounded derivatives.
This choice guarantees smoothness of the MEL model, which is critical for stable ADMM updates.

In a nutshell, the proposed neural network based framework for learning the ME $M_{\frac{1}{\rho}}f$ is summarized in Table~\ref{tab:ICNN}.

\begin{table}[!tb]
\centering
{\footnotesize
\begin{tblr}{|c|c|c|}
\hline
Function $f$ & ME $M_{\frac{1}{\rho}} f$ & Neural network model
\\
\hline
Convex & Convex, $\rho$-Lipschitz continuous gradient & $\ICNN_\theta$
\\
\cline{1-3}
Strongly convex & Strongly convex, $\rho$-Lipschitz continuous gradient & $\ICNN_\theta $+ $\frac{\sigma_M}{2} \Vert.\Vert^2_2$
\\
\hline
\end{tblr}
\caption{Configurations for learning the Moreau envelope of a convex function.
Activation functions should be non-decreasing, smooth, and its first and second derivatives are bounded}
\label{tab:ICNN}
}
\end{table}

\begin{remark}
If a function is $\sigma$-strongly convex, then it is also $\sigma^\prime$-strongly convex for all $0 < \sigma^\prime \leq \sigma$.
This allows us to determine the strong convexity modulus $\sigma_M$ of the ME.
In particular, the largest $\sigma_M > 0$ such that $M_{\frac{1}{\rho}}f - \frac{\sigma_M}{2}\Vert .\Vert_2^2$ is convex will be used.
\end{remark}

\subsection{Training the MEL model}

\begin{algorithm}[t]
\caption{The MEL-ADMM algorithm}
\label{alg:ADMM-MEL}
\begin{algorithmic}[1]
\Require Parameters $\rho > 0$, $\gamma \in (0, 2)$,
initial points $w^0$ and $\alpha^0$,
and trained MEL model $\hat{M}_{\frac{1}{\rho}}f$
\Repeat
\State{$\hat z^{i+1} = w^i + \rho^{-1} \alpha^i - \frac{1}{\rho} \nabla \hat{M}_{\frac{1}{\rho}}f (w^i + \rho^{-1} \alpha^i)$}
\State{$\tilde z^{i+1} =  \gamma \hat z^{i+1} + (1-\gamma)w^i$}
\State{$w^{i+1} = \Pi_{\bar \calZ}(\tilde z^{i+1} - \rho^{-1} \alpha^i)$}
\State{$\alpha^{i+1} = \alpha^i + \rho(w^{i+1} - \tilde z^{i+1})$}
\Until The termination criterion is satisfied
\State \Return $w^{i+1}$
\end{algorithmic}
\end{algorithm}

To train the MEL model, we construct a dataset  $\{(q^{i}, M_{\frac{1}{\rho}}f(q^{i}), \nabla M_{\frac{1}{\rho}}f(q^{i}))\}_{i=1,\dots,n_D}$ obtained from the ADMM iteration \eqref{eq:ADMM-gen}, where each sample consists of a query point $q_i = w^i + \rho^{-1} \alpha^i$, the corresponding ME value $M_{\frac{1}{\rho}}f(q^{i})$, and its gradient $\nabla M_{\frac{1}{\rho}}f(q^{i})$.
The ME value and its gradient are determined by the minimum value of \eqref{eq:ADMM-gen-z} and by \eqref{eq:MErelationship}.
To systematically collect the data across various initial conditions, we sample the initial points $w^0 \in \calZ$ and $\alpha^0 \in \bbR^n$ using normal distributions, then execute the iteration \eqref{eq:ADMM-gen} to obtain the data.
The data collection is performed offline, allowing us to gather as much data as needed to train the MEL model.

The training loss function consists of the output loss term and a gradient loss term for matching the ME gradient values $\nabla M_{\frac{1}{\rho}}f(q^{i})$ in the training data
\begin{align}
\label{eq:loss-function-out}
\calL_{\mathrm{out}} &= \frac{1}{n_D}\sum_{i=1}^{n_D} \left( \hat{M}_{\frac{1}{\rho}}f(q^i) - M_{\frac{1}{\rho}}f(q^i) \right)^2,
\\
\label{eq:loss-function-grad}
\calL_{\mathrm{grad}} &= \frac{1}{n_D}\sum_{i=1}^{n_D} \Big( \mu_g \left\| \nabla \hat{M}_{\frac{1}{\rho}}f(q^i) -
\nabla M_{\frac{1}{\rho}}f(q^i) \right\|_2^2 + \mu_p \Vert \max(0, \hat{M}_{\frac{1}{\rho}}f(q^i) - f(q^i)) \Vert_2^2 \Big)\text.
\end{align}
In \eqref{eq:loss-function-grad}, the weights $\mu_g$ and $\mu_p$ are positive, and the second term $\| \max(0,\cdot)\|_2^2$ is designed to enforce property (ii) in Proposition~\ref{pro:pros} during the training process.

Putting everything together, the MEL model $\hat{M}_{\frac{1}{\rho}}f$ is trained by minimizing the total loss function
\(\minimize_\theta \; \calL_{\mathrm{out}} + \calL_{\mathrm{grad}} \)
subject to $W_i$ being element-wise nonnegative.

\begin{remark}
If $f$ is differentiable and finite, the MEL model can be trained without labels for $\nabla M_{\frac{1}{\rho}}f$.
For each training input $q^i$, define
$\calL_{\mathrm{ME}} = \frac{1}{n_D}\sum_{i=1}^{n_D} \left( f(\hat q^i)+\frac{1}{2\rho}\|\nabla\hat M_{\frac{1}{\rho}}f(q^i)\|_2^2\right)$ where
$\hat q^i = q^i - \frac{1}{\rho}\nabla\hat M_{\frac{1}{\rho}}f(q^i)$, so that the second term equals $\frac{\rho}{2}\|\hat q^i-q^i\|_2^2$.
An alternative training objective is then $\calL_{\mathrm{out}}+\omega\calL_{\mathrm{ME}}$, with $\omega>0$.
For each $q^i$, $\calL_{\mathrm{ME}}$ is minimized at $\prox_{\frac{1}{\rho}f}(q^i)$; hence it encourages $\hat q^i$ to approximate the proximal point.
Differentiating this loss through $f(\hat q^i)$ avoids generating ME-gradient labels and can reduce offline data-preparation cost.
\end{remark}

\subsection{The MEL-ADMM algorithm}
\label{sec:MEL-ADMM}

Once the MEL model $\hat{M}_{\frac{1}{\rho}}f$ is trained, it is integrated into the ADMM iteration \eqref{eq:ADMM-gen} by replacing the $z$-update in \eqref{eq:ADMM-gen-z} with
\begin{equation}
\hat z^{i+1} = q^i - \rho^{-1} \nabla \hat{M}_{\frac{1}{\rho}}f (q^i)\text.
\label{eq:learning-ADMM}
\end{equation}
Here we use $\hat z$ to distinguish the predicted decision vector $z$ from the exact one.
The resulting algorithm is termed MEL-ADMM and is presented in Algorithm~\ref{alg:ADMM-MEL}.
It returns the auxiliary variable $w$, which lies in $\bar\calZ$ at every iteration because of the projection step.

Compared to the ADMM iteration \eqref{eq:ADMM-gen}, the MEL-ADMM has a significant computational advantage.
Specifically, instead of solving the potentially expensive optimization problem \eqref{eq:ADMM-gen-z} of the $z$-update, the $\hat z$-update in \eqref{eq:learning-ADMM} of MEL-ADMM is an inexpensive forward computation (inference) of the neural network model $\hat{M}_{\frac{1}{\rho}}f$, thus the computation time can be shortened.
In addition, the MEL model is independent of the problem instance, therefore, once trained, MEL-ADMM can be applied to solve any instance of \eqref{eq:general_opti} without retraining the MEL model.

Observe that $\bar\calZ$ incorporates both a convex set $\calZ$ and an affine equality constraint, hence the projection onto $\bar\calZ$ in \eqref{eq:ADMM-gen-w} at each iteration can be computationally expensive.
The MEL-ADMM algorithm can be further accelerated by treating these constraints separately, as presented next.

\subsection{The splitting MEL-ADMM (sMEL-ADMM) algorithm}
\label{sec:sMEL-ADMM}

The constraints \eqref{eq:gen_opti_eq_cstr} and \eqref{eq:gen_opti_ieq_str} can be split into two sets by introducing another auxiliary variable $v \in \bbR^n$, leading to an equivalent form of problem \eqref{eq:general_opti} as
\begin{subequations}
\label{eq:general_opti_con_split}
\begin{optimization}{f(z) + I_{\calZ}(w) + I_{\calA}(v)}
&w - z = 0, \\
&w - v = 0,
\end{optimization}
\end{subequations}
where $\calA = \{z \in \bbR^n | Az = b \}$.
The augmented Lagrangian of \eqref{eq:general_opti_con_split} is expressed as
\begin{align*}
L(z, w, v, \alpha, \beta) =  f(z)
+ I_{\calZ} (w) + I_{\calA} (v) &+ \alpha^\top(w - z) + \frac{\rho}{2}  \Vert w - z \Vert^2_2
\\
&+ \beta^\top(w - v) + \frac{\rho}{2}  \Vert w - v \Vert^2_2\text.
\end{align*}
Inspired by operator splitting methods \citep{Stellato2018OSQP} and based on the reformulation \eqref{eq:general_opti_con_split}, we propose a method named splitting MEL-ADMM (sMEL-ADMM) in the following form
\begin{subequations} \label{eq:admm-update_split}
\begin{align}
z^{i+1} &= \argmin f(z) + \frac{\rho}{2} \Vert w^i - z + \rho^{-1}\alpha^i \Vert_2^2
= \prox_{\frac{1}{\rho} f} \left( w^i + \rho^{-1} \alpha^i\right)
\label{eq:admm_z-update_split}
\\
\tilde z^{i+1} &= \gamma z^{i+1} + (1-\gamma) w^i,
\label{eq:admm_tz-update_split}
\\
w^{i+1} & = \argmin_{w \in \calZ } \Vert w \!-\! \tilde z^{i+1} \!+\! \rho^{-1} \alpha^i\Vert_2^2 \!+\!
\Vert w \!-\! v^i  \!+\! \rho^{-1} \beta^i \Vert_2^2 \nonumber
\\
&= \Pi_{\calZ } \left(\frac{\tilde z^{i+1} +  v^i - \rho^{-1} \alpha^i  - \rho^{-1} \beta^i}{2} \right),
\label{eq:admm_w-update_split}
\\
v^{i+1} & = \Pi_{\calA} \left( w^{i+1} + \rho^{-1} \beta^i  \right),
\label{eq:admm_v-update_split}
\\
\alpha^{i+1} &= \alpha^i + \rho(w^{i+1} - \tilde z^{i+1}),
\label{eq:admm_al-update_split}
\\
\beta^{i+1} &= \beta^i + \rho (w^{i+1} - v^{i+1}).
\label{eq:admm_be-update_split}
\end{align}
\end{subequations}
Note that the $z$-update \eqref{eq:admm_z-update_split} can be replaced by the MEL model as in MEL-ADMM.
Since $\calA$ is an affine set, the $v$-update \eqref{eq:admm_v-update_split} admits a closed-form solution by directly solving the KKT conditions as follows:
\begin{align}\label{eq:closed-form-v-update}
\begin{bmatrix} I &A^\top \\ A &0\end{bmatrix} \begin{bmatrix} v^{i+1} \\ \lambda^{i+1} \end{bmatrix} = \begin{bmatrix} w^{i+1} + \rho^{-1}\beta^i \\ b \end{bmatrix},
\end{align}
where $\lambda^{i+1} \in \bbR^m$ is the Lagrange multiplier associated with the equality constraint $Av = b$.
The matrix on the left-hand side can be decomposed, \eg using LU decomposition, prior to the ADMM iterations, allowing the linear system \eqref{eq:closed-form-v-update} to be solved efficiently at each iteration.

Employing the MEL model $\hat{M}_{\frac{1}{\rho}}f$ in \eqref{eq:admm_z-update_split} and the closed-form $v$-update in \eqref{eq:closed-form-v-update}, the sMEL-ADMM algorithm is presented in Algorithm~\ref{alg:ADMM-MEL-split}.
In practice, the set $\calZ$ is often simple, such as a box\rev{,} allowing an analytical form of the $w$-update \eqref{eq:admm_w-update_split} to be obtained.
In such a case, Algorithm~\ref{alg:ADMM-MEL-split} completely eliminates the need for an optimization solver, which significantly accelerates the computation.

\begin{algorithm}[t]
\caption{The sMEL-ADMM algorithm}
\label{alg:ADMM-MEL-split}
\begin{algorithmic}[1]
\Require Parameters $\rho > 0$, $\gamma \in (0, 2)$,
initial points $w^0$, $v^0$, $\alpha^0$, $\beta^0$,
and trained MEL model $\hat{M}_{\frac{1}{\rho}}f$.

\Repeat
\State{$ \hat z^{i+1} = w^i + \rho^{-1} \alpha^i - \frac{1}{\rho} \nabla \hat{M}_{\frac{1}{\rho}}f (w^i + \rho^{-1} \alpha^i)$}
\State{$\tilde z^{i+1} = \gamma \hat z^{i+1} + (1-\gamma) w^i$}
\State{$w^{i+1}, v^{i+1}$, $\alpha^{i+1}$, and $\beta^{i+1}$
by \eqref{eq:admm_w-update_split}-\eqref{eq:admm_be-update_split}, \eqref{eq:closed-form-v-update}}
\Until{Termination criterion is satisfied}
\State \Return $v^{i+1}$
\end{algorithmic}
\end{algorithm}

\section{Theoretical Analysis}
\label{sec:analysis}

In this section, for analyzing the convergence of the (s)MEL-ADMM algorithms, we consider the MEL model $\hat{M}_{\frac{1}{\rho}} f$ as an ICNN, $\ICNN_\theta$, parameterized by $\theta$ as in \eqref{eq:ICNN}.
To prove Lipschitz continuity of $\nabla \ICNN_{\theta}(x)$, we consider $\ICNN_{\theta}(x)$ in the following form, which is \eqref{eq:ICNN} with the layer parameters re-indexed,
\begin{align*}
y_0 &= W_0 x + b_0,\;
v_0 = \phi_0(y_0) \quad \text{(element-wise)},
\\
y_j &= W_j v_{j-1} + V_j x + b_j,\;
v_j = \phi_j(y_j) \quad \text{(element-wise)},\;j=1,\dots,\ell, \\
h(x) &= \ICNN_\theta(x)=\phi_\ell(y_\ell).
\end{align*}

\begin{lemma}
\label{lem:LipschitzNN}
The ICNN \eqref{eq:ICNN} with finite weights has Lipschitz-continuous gradient if all activation functions $\phi_j$ are twice continuously differentiable and their derivatives
$|\phi_j^{\prime}|$ and $|\phi_j^{\prime\prime}| $ are bounded (\eg softplus or squareplus, which are convex and non-decreasing as required by the ICNN).
\end{lemma}
\noindent {\bf Proof}: See Appendix \ref{appx:lipschitz}.

Before stating the main theorem on convergence, we establish the following useful preliminary result.

\begin{lemma}
\label{thm:convergence1}
Let $g$ be a convex differentiable function, whose gradient is Lipschitz continuous with constant $\frac{1}{\lambda}$.
There exists a unique  $ f \in \Gamma_0(\bbR^n)$ such that $M_{\lambda}f = g$.
\end{lemma}

\noindent{\bf Proof}: See Appendix \ref{appx}.

To prove the convergence of Algorithm~\ref{alg:ADMM-MEL}, let us consider its scaled form as follows, with $\bar\alpha^i =  \frac{1}{\rho} \alpha^i$,
\begin{subequations}
\label{eq:ADMM_gen_xi}
\begin{align}
\hat z^{i+1} &= w^i + \bar \alpha^i - \rho^{-1} \nabla \ICNN_\theta (w^i + \bar \alpha^i),
\label{eq:ADMM_gen_xi-z}
\\
\tilde z^{i+1} &=  \gamma \hat z^{i+1} + (1-\gamma)w^i,
\\
w^{i+1} &= \Pi_{\bar \calZ}(\tilde z^{i+1} - \bar \alpha^i),
\\
\bar \alpha^{i+1} &=\bar \alpha^i +  w^{i+1} - \tilde z^{i+1}\text.
\end{align}
\end{subequations}
Define $g(x) = \ICNN_\theta(x)$.
By Lemma \ref{lem:LipschitzNN}, $\nabla g(x)$ is Lipschitz continuous.
Let $L$ be a Lipschitz constant of $\nabla g$ and suppose that $L \leq \rho$.
Lemma~\ref{thm:convergence1} guarantees a unique function $\hat f \in \Gamma_0(\bbR^n)$ such that $M_{\frac{1}{\rho}}\hat f=g$ because $\nabla g$ is also $\rho$-Lipschitz continuous.
We then define the following optimization problem
\begin{equation}
\label{eq:equivalent}
\minimize \; \hat{f}( z)
\;\; \text{s.t.} \;\;
\eqref{eq:gen_opti_eq_cstr}, \eqref{eq:gen_opti_ieq_str}\text.
\end{equation}
The feasibility of the constraint set has already been assumed in Section~\ref{sec:general_opti}.
Since $\calZ$ is closed, so is $\bar{\calZ} = \calZ \cap \calA$, where $\calA = \{z \in \bbR^n \,|\, Az = b\}$.
Here, $\bar{\calZ}$ is the feasible set defined after \eqref{eq:general_opti_con}.
We additionally assume that the problem \eqref{eq:equivalent} admits a primal-dual solution, that is, there exist $\hat z^\star\in\bar{\calZ}$ and $\hat\alpha^\star\in\bbR^n$ such that
\begin{equation}
\label{eq:surrogate-saddle-point}
\hat\alpha^\star\in\partial\hat f(\hat z^\star),
\qquad -\hat\alpha^\star\in N_{\bar{\calZ}}(\hat z^\star).
\end{equation}
Here $\partial\hat f$ is the convex subdifferential of $\hat f$ and
$N_{\bar{\calZ}}(z)=\{v\in\bbR^n\mid v^\top(y-z)\leq 0\ \text{for all }y\in\bar{\calZ}\}$ is the normal cone at $z\in\bar{\calZ}$.
Condition~\eqref{eq:surrogate-saddle-point} is the saddle-point assumption for the (unaugmented) Lagrangian of the consensus problem \eqref{eq:general_opti_con} with $f = \hat f$.


The following theorem establishes the convergence of MEL-ADMM to a solution of \eqref{eq:equivalent}.

\begin{theorem}
\label{thm:cvg_MEL-ADMM}
If $L \leq \rho$  and \eqref{eq:surrogate-saddle-point} holds, the MEL-ADMM Algorithm \ref{alg:ADMM-MEL} for \eqref{eq:general_opti} converges to an optimal solution of \eqref{eq:equivalent}.
\end{theorem}

\noindent {\bf Proof}: See Appendix~\ref{appx:cvg-MEL-ADMM}.

Theorem~\ref{thm:cvg_MEL-ADMM} guarantees convergence for the learned objective $\hat f$.
We next quantify, under additional assumptions on $f$, the accuracy of the limiting solution for the original objective.
Let $F=M_{\frac{1}{\rho}}f$ and assume that, for some $\varepsilon_g\geq 0$,
\begin{equation}
\label{eq:uniform-ME-gradient-error}
\|\nabla g(q)-\nabla F(q)\|_2\leq\varepsilon_g,
\qquad \text{for all } q\in\calQ,
\end{equation}
where $\calQ\subseteq\bbR^n$ is a closed set containing all query points $q^i=w^i+\rho^{-1}\alpha^i$.
By Theorem~\ref{thm:cvg_MEL-ADMM}, $\calQ$ then also contains their limit
$\hat q^\star=\hat z^\star+\rho^{-1}\hat\alpha^\star$, which is the only point at which \eqref{eq:uniform-ME-gradient-error} is used in the proofs below.

\begin{theorem}[Bounded stationarity and objective errors]
\label{thm:bounded-optimality-error}
\setlength{\emergencystretch}{1em}
Suppose Theorem~\ref{thm:cvg_MEL-ADMM} applies and \eqref{eq:uniform-ME-gradient-error} holds at the limit $q = \hat q^\star$.
Assume additionally that $f:\bbR^n\to\bbR$ is differentiable with $L_f$-Lipschitz gradient and that the original problem attains its minimum at some $z^\star\in\bar{\calZ}$.
Define $\eta=(1+L_f/\rho)\varepsilon_g$.
Then the MEL-ADMM limit $\hat z^\star\in\bar{\calZ}$ satisfies
\begin{equation}
\label{eq:bounded-stationarity-error}
\operatorname{dist}\big(0,\nabla f(\hat z^\star)+N_{\bar{\calZ}}(\hat z^\star)\big)
\leq\eta,
\end{equation}
and
\begin{equation}
\label{eq:bounded-objective-error}
0\leq f(\hat z^\star)-f(z^\star)
\leq\eta\|\hat z^\star-z^\star\|_2.
\end{equation}
\end{theorem}
\noindent {\bf Proof}: See Appendix~\ref{appx:bounded-optimality-error}.

Theorem~\ref{thm:bounded-optimality-error} bounds the error in the optimal value of the original problem.
If, in addition, $f$ is strongly convex, the error in the optimal solution $z^\star$ itself can also be bounded, as stated in the following proposition.

\begin{samepage}
\begin{proposition}[Error bounds under strong convexity]
\label{prop:strongly-convex-error}
Under the assumptions of Theorem~\ref{thm:bounded-optimality-error}, suppose additionally that $f$ is $\mu$-strongly convex, with $\mu>0$.
Then the original problem has a unique optimizer $z^\star$ 
, and
\begin{align}
\|\hat z^\star-z^\star\|_2
&\leq\frac{(1+L_f/\rho)\varepsilon_g}{\mu},
\label{eq:strongly-convex-solution-error}\\*
0\leq f(\hat z^\star)-f(z^\star)
&\leq\frac{(1+L_f/\rho)^2\varepsilon_g^2}{2\mu}.
\label{eq:strongly-convex-objective-error}
\end{align}
These bounds do not require $\bar{\calZ}$ to be bounded.
\end{proposition}
\end{samepage}
\noindent {\bf Proof}: See Appendix~\ref{appx:strongly-convex-error}.

\begin{remark}
The uniform gradient bound \eqref{eq:uniform-ME-gradient-error} is an additional approximation assumption; a small empirical training loss does not by itself certify this bound over $\calQ$.
Theorem~\ref{thm:bounded-optimality-error} and Proposition~\ref{prop:strongly-convex-error} require the stated smoothness of $f$.
For MEL-ADMM, exact projection ensures $w^{i+1}\in\bar{\calZ}$ at every iteration; feasibility of $\hat z^{i+1}$ follows asymptotically from consensus.
The objective and solution bounds concern the converged solution and do not assert a finite-iteration convergence rate.
\end{remark}

For sMEL-ADMM, the surrogate of \eqref{eq:general_opti_con_split} is obtained by replacing $f$ with $\hat f$, and the corresponding saddle-point optimality condition is that there exist $\hat z^\star\in \bar\calZ = \calZ\cap\calA$ and $\hat\alpha^\star,\hat\beta^\star\in\bbR^n$ such that
\begin{equation}
\label{eq:surrogate-saddle-point-split}
\hat\alpha^\star\in\partial\hat f(\hat z^\star),
\qquad -\hat\alpha^\star-\hat\beta^\star\in N_{\calZ}(\hat z^\star),
\qquad \hat\beta^\star\in N_{\calA}(\hat z^\star).
\end{equation}

\begin{proposition}
\label{prop:sMEL-ADMM}
If $L < \rho$ and \eqref{eq:surrogate-saddle-point-split} holds, the sMEL-ADMM Algorithm \ref{alg:ADMM-MEL-split} converges to an optimal solution of \eqref{eq:equivalent}.
\end{proposition}

\noindent {\bf Proof}: As in the proof of Theorem~\ref{thm:cvg_MEL-ADMM}, $M_{\frac{1}{\rho}}\hat f=g$ and \eqref{eq:MErelationship} imply $\hat z^{i+1}=\prox_{\frac{1}{\rho}\hat f}(w^i+\rho^{-1}\alpha^i)$.
Hence, with $\gamma=1$, Algorithm~\ref{alg:ADMM-MEL-split} is the direct three-block extension of ADMM, with the update order $z\to w\to v$ and a common penalty $\rho$, applied to
\[
\minimize\; \hat f(z)+I_{\calZ}(w)+I_{\calA}(v)
\;\;\text{s.t.}\;\; A_1 z + A_2 w + A_3 v = 0,\]
where
\[
A_1=\begin{bmatrix}-I\\0\end{bmatrix},\;
A_2=\begin{bmatrix}I\\I\end{bmatrix},\;
A_3=\begin{bmatrix}0\\-I\end{bmatrix},
\]
with multiplier $(\alpha,\beta)$ (up to the sign convention).
The direct extension of ADMM is not convergent in general for three blocks \citep{chen2016direct}.
Here, however, $A_1^\top A_3=0$, the matrix $[A_1\;A_2]=\begin{bmatrix}-I&I\\0&I\end{bmatrix}$ is nonsingular, and $A_3$ has full column rank.
The functions $\hat f$, $I_{\calZ}$, and $I_{\calA}$ are closed, proper, and convex, and \eqref{eq:surrogate-saddle-point-split} is precisely the KKT condition of this problem.
Therefore, Theorem~2.4 of \citet{chen2016direct} implies that the iterates converge to a KKT point, at which $\hat z^\star=w^\star=v^\star\in\calZ\cap\calA=\bar{\calZ}$ minimizes $\hat f$ over $\bar{\calZ}$, \ie solves \eqref{eq:equivalent}.
\hfill$\square$

\begin{remark}
The analysis in \citet{chen2016direct} covers only $\gamma=1$.
Consequently, Proposition~\ref{prop:sMEL-ADMM} does not apply to the over-relaxed setting $\gamma\in(1,2)$, including $\gamma=1.6$ used in all sMEL-ADMM experiments in Section~\ref{sec:experiments}.
For those experiments, convergence is observed numerically;
extending the guarantee to $\gamma > 1$ is left for future work.
\end{remark}

\begin{remark}[Feasibility]
\label{rem:feasibility}
(s)MEL-ADMM converges to a solution of the surrogate problem \eqref{eq:equivalent}, not necessarily of \eqref{eq:general_opti}, and its feasibility guarantees are asymptotic; at termination, the returned solution is approximately feasible.
MEL-ADMM returns $w^i$, which lies in $\bar\calZ$ at every iteration because of the exact projection onto $\bar\calZ$; this projection, however, is itself an optimization problem, which motivates sMEL-ADMM.
sMEL-ADMM returns the auxiliary variable $v^i$ and applies no final projection.
Since $v^i\in\calA$, the equality constraints \eqref{eq:gen_opti_eq_cstr} hold at every iteration, whereas $v^i$ need not belong to $\calZ$.
Because $w^i\in\calZ$ and, by \eqref{eq:admm_be-update_split}, $w^i-v^i=\rho^{-1}(\beta^i-\beta^{i-1})$, we have
\[
\operatorname{dist}(v^i,\calZ)\leq\Vert v^i-w^i\Vert_2=\rho^{-1}\Vert\beta^i-\beta^{i-1}\Vert_2\text.
\]
Therefore, the violation of \eqref{eq:gen_opti_ieq_str} is bounded by the primal residual of the $w$--$v$ consensus, which vanishes under the conditions of Proposition~\ref{prop:sMEL-ADMM}.
At termination, $v^i$ is therefore approximately feasible, with a violation controlled by the residual tolerance; Section~\ref{sec:experiments} reports the resulting constraint violations for the experiments.
\end{remark}

For MEL-ADMM, Theorem~\ref{thm:bounded-optimality-error} and Proposition~\ref{prop:strongly-convex-error} quantify the original objective and solution errors under a uniform envelope-gradient approximation bound and the stated smoothness and strong-convexity assumptions.
The constants in these bounds differ in how they can be obtained.
The constants $L_f$, $\mu$, and $D$ depend only on $f$ and $\bar\calZ$ and can be computed or bounded analytically for a given problem.
The approximation error $\varepsilon_g$ can be estimated numerically by evaluating $\Vert\nabla g(q) - \nabla F(q)\Vert_2$ on sampled query points, but such an estimate does not certify the bound \eqref{eq:uniform-ME-gradient-error}.
The objectives in the experiments in Section~\ref{sec:experiments} are non-smooth or not finite everywhere and are mostly solved with sMEL-ADMM, so the experiments do not evaluate these bounds directly; they instead report the achieved optimality gaps and constraint violations.

\begin{remark}
Lipschitz-constant estimators such as those in \citet{virmaux2018lipschitz, fazlyab2019efficient} bound the Lipschitz constant of $\ICNN_\theta$ itself, \ie $\sup_x\Vert\nabla\ICNN_\theta(x)\Vert_2$, and therefore do not directly bound the Lipschitz constant $L$ of $\nabla\ICNN_\theta$.
A certified upper bound $L\leq\sup_x\Vert\nabla_x^2\ICNN_\theta(x)\Vert_2$ can instead be computed layer by layer from the weight norms and the bounds on $|\phi_j'|$ and $|\phi_j''|$, following the recursion in Appendix~\ref{appx:lipschitz}; this bound is conservative but suffices to verify the condition $L \leq \rho$ in the above results.
\end{remark}

\section{LEAF for Model Predictive Control}
\label{sec:MPC}

The (s)MEL-ADMM algorithms in the LEAF framework, presented in Section~\ref{sec:framework}, can be applied to accelerate model predictive control (MPC) for linear time-invariant (LTI) systems.
Consider the following discrete-time LTI system
\begin{equation}
\label{eq:LTI system}
x_{t + 1} = A{x_t} + B{u_t} \text,
\end{equation}
where $x_t \in \calX \subseteq \bbR^{n_x}$ is the state,
$u_t \in \calU \subseteq \bbR^{n_u}$ is the control input,
real matrices $A$ and $B$ are of appropriate dimensions, and $t$ is the time step.
MPC \citep{mayne2014model} computes a feedback control input by solving the constrained optimization problem
\begin{subequations}
\label{eq:MPC}
\begin{optimization}{J(\bfx, \bfu ) = c_N(x_{N|t}) + \sum_{k=0}^{N-1} c_k(x_{k|t}, u_{k|t})\label{eq:MPC cost function}}
&x_{k + 1|t} = A x_{k|t} \!+\!
B u_{k|t}\label{eq:dynamics_constraint}, \; x_{0|t} = x_t,\\
&(x_{k|t}, u_{k|t}) \in \mathcal{X} \times \mathcal{U},\;k = 0, \dots, N \!-\! 1\text, \label{eq:box_constraint}
\end{optimization}
\end{subequations}
where $N \in \bbN_{>0}$ is the prediction horizon,
$x_{k|t}$ and  $u_{k|t}$ are the predicted state and input at step $k$ in the horizon,
$c_k(x_{k|t},u_{k|t})$ is the stage cost function,
$c_N(x_{N|t})$ is the terminal cost function, and $x_{0|t} = x_t$ is the current state.
The constraint sets $\calX$ and $\calU$ are convex.
In addition, the cost functions $c_k$ and $c_N$ are convex, proper, and closed.
The MPC problem is parameterized by $x_t$.

Let us define
$z_k = [ x_{k|t}^\top, u_{k|t}^\top ]^\top$, for $k=0,\dots, N-1$,  and $z_N = [ x_{N|t}^\top, 0^\top ]^\top$.
Then, the problem \eqref{eq:MPC} is reformulated as
\begin{subequations}
\label{eq:MPC distributed}
\begin{optimization}[\minimize_{z_0, \dots, z_N}]
{\sum_{k=0}^N (f_k(z_k) + I_{\mathcal{Z}_k} (z_k))\label{eq:MPC_z}}
&M_1 z_0 = x_t,\\
&M_1 z_{k+1} - M_2 z_k = 0,\;k=0, \dots, N-1\text,
\end{optimization}
\end{subequations}
where $f_N(z_N) = c_N(x_{N|t})$, $\calZ_N = \calX \times \{0\}$,
$f_k(z_k) = c_k(x_{k|t},u_{k|t})$ and ${\calZ}_k = {\calX} \times {\calU}$ for $k=0,\dots,N-1$,
$M_1 = [I_{n_x}, 0_{n_x \times n_u}]$, and $M_2 = [A, B]$.
By defining auxiliary variables $w_k \in \bbR^{n_x+n_u}$ that are copies of $z_k$, the problem~\eqref{eq:MPC distributed} can be rewritten in consensus form as
\begin{subequations}
\label{eq:MPC_consen}
\begin{optimization}{J(z, w) = \sum_{k=0}^{N} f_k(z_k) + I_{\mathcal{Z}} (w)\label{eq:MPC_z_w}}
&w - z = 0\text,
\end{optimization}
\end{subequations}
where $z = [z_k]_{k=0,1, \dots, N}$,  $w = [w_k]_{k=0,1, \dots, N}$, and
$\mathcal{Z} = \{w \in \calZ_0 \times \dots \times \calZ_N \,|\, M_1 w_0 = x_t,\; M_1 w_{k+1} - M_2 w_k = 0 \; \text{for $k = 0,\dots,N-1$}\}$.

Applying MEL-ADMM in Section~\ref{sec:MEL-ADMM} to \eqref{eq:MPC_consen} yields
\begin{subequations}
\label{eq:ADMM-PCA}
\begin{align}
\hat z_k^{i+1} &= q_k^i - \rho^{-1} \nabla \hat{M}_{\frac{1}{\rho}} f_k (q_k^i), \; k = 0, 1, \cdots, N,
\label{eq:ADMM-PCA-z}
\\
\tilde z^{i+1} &= \gamma \hat z^{i+1} + (1-\gamma) w^i,
\\
w^{i+1} &= \Pi_{\calZ} \left(\tilde z^{i+1} - \rho^{-1} \alpha^i \right),
\label{eq:ADMM-PCA-w}
\\
\alpha^{i+1} &= \alpha^i + \rho (w^{i+1} - \tilde z^{i+1}),
\label{eq:ADMM-PCA-b}
\end{align}
\end{subequations}
where $q_k^i = w_k^i + \rho^{-1}\alpha_k^i$ and $\hat{M}_{\frac{1}{\rho}} f_k$ is the MEL model of $f_k$.
In \eqref{eq:ADMM-PCA}, the $\hat z$-update, $\tilde z$-update, and $\alpha$-update can be carried out in a parallel manner by $k$, and the $w$-update \eqref{eq:ADMM-PCA-w} is a Euclidean projection onto the convex set $\calZ$.
If the stage cost functions $c_k$ are identical, hence the functions $f_k$ and their Moreau envelopes are also identical, for all $k=0,\dots,N-1$, then \eqref{eq:ADMM-PCA-z} is simplified as it requires only two MEL models $\hat{M}_{\frac{1}{\rho}} f_k$ and $\hat{M}_{\frac{1}{\rho}} f_N$.

The MPC problem can also be solved by adopting sMEL-ADMM in Section~\ref{sec:sMEL-ADMM}.
Specifically, define two constraint sets $\calC_\text{bound} = (\calX\times \calU)^N \times (\calX \times \{0\})$ and $\calC_\text{dynamic} = \{z \in \bbR^{(n_x + n_u)(N+1)} \,|\, \bfM z = d_t \}$, where $\bfM$ is the $(N+1)\times (N+1)$-block matrix
\[ \bfM = \left[ \begin{array}{rrrr} M_1 \\\!\!\!-M_2 &M_1\\ &\ddots &\ddots \\ &&\!\!\!-M_2 &M_1  \\ \end{array} \right],~\text{and}~
d_t = [x_t^\top, 0^\top]^\top.\] Then \eqref{eq:MPC distributed} can be converted into \eqref{eq:general_opti_con_split} with $\calZ = \calC_{\mathrm{bound}}$ and $\calA = \calC_{\mathrm{dynamic}}$.

LEAF enables computationally efficient solving of the MPC problem \eqref{eq:MPC}, especially when the cost functions are convex but non-smooth, as demonstrated in Section~\ref{sec:experiments}.

\section{Numerical Results}
\label{sec:experiments}
This section validates the proposed framework through four numerical examples.
The first example considers convex quadratic programming.
The second studies energy management in a microgrid with renewable generation and a battery energy storage system.
The third example considers entropy maximization problem.
The source code for all the experiments are available at \url{https://github.com/nxt-lab/2026-LEAF_Moreau-Envelop-Learning-ADMM.git}

\subsection{Experimental setups}

The proposed (s)MEL-ADMM methods are designed to obtain high-quality solutions with low computational cost, rather than to drive the objective value arbitrarily close to the exact optimum.
Their accuracy is therefore evaluated through the achieved optimality gap and constraint violation.
Their feasibility guarantees are asymptotic, and feasibility at termination is approximate.
Specifically, sMEL-ADMM returns the auxiliary variable $v$, which satisfies the equality constraints by construction, while its violation of the remaining constraints is bounded by the $w$--$v$ consensus residual (Remark~\ref{rem:feasibility}); no final projection is applied.
Consequently, the comparisons below focus on accuracy, constraint violation, and computational time.

\subsubsection{Termination condition}
Termination criteria differ across optimization methods, particularly between interior-point methods (IPMs) and operator-splitting methods.
For example, IPOPT \citep{wachter2006implementation} and MadNLP \citep{pacaud2024gpu} use termination tolerances based on KKT conditions, whereas OSQP \citep{Stellato2018OSQP} and other ADMM implementations use criteria based on primal and dual residuals.
Therefore, for each example, we choose termination criteria that support a fair comparison across methods.
Specifically, we use the \textit{relative optimality gap}:
\begin{align}
\label{eq:opt_gap}
g_{opt} = \frac{ | J - J^\star|}{\max(|J^\star|,\varepsilon)}  \times 100\%,
\end{align}
where $J$ is the objective value of the solution returned by a solver, $J^\star$ is the optimal objective value, obtained using high-accuracy IPM configurations, and $\varepsilon$ is a small positive constant.

The constraint violation of a point $x$ with respect to equality constraint $c_\mathrm{eq}(x) = 0$ and inequality constraint $c_\mathrm{ineq}(x) \leq 0$ is measured by
\begin{equation}
\label{eq:cstr_viol}
c_v(x) = \max\big(\Vert c_\mathrm{eq}(x)\Vert_\infty,\ \Vert [c_\mathrm{ineq}(x)]_+\Vert_\infty\big),
\end{equation}
where $[\cdot]_+ = \max(0,\cdot)$ is applied element-wise.
For (s)MEL-ADMM, $c_v$ is evaluated at the returned auxiliary variable, $w$ for MEL-ADMM and $v$ for sMEL-ADMM, as in Algorithms~\ref{alg:ADMM-MEL} and~\ref{alg:ADMM-MEL-split}, without any final projection.

We then tune the termination tolerances so that all optimizers return solutions within the same prescribed optimality gap and constraint violation.
The target constraint violation is set to $c_v \leq 10^{-4}$ for all examples, and the target optimality gap is set to $g_{opt} \leq 1\%$, $g_{opt} \leq 0.1\%$, and/or $g_{opt} \leq 0.01\%$ depending on each example.

\subsubsection{Implementation and baselines}
We compare the computational time of (s)MEL-ADMM against several state-of-the-art solvers, including: IPOPT \citep{wachter2006implementation}, MadNLP \citep{pacaud2024gpu}, Mosek \citep{andersen2000mosek}, and Clarabel \citep{goulart2024clarabel}.
We also compare our methods with ADMM \citep{boydDistributedOptimization2011}, implemented as in our methods but without using the MEL model.
In addition, we compare (s)MEL-ADMM to DC3 \citep{dontiDC3Learning2021}, a well-known learning-to-optimize method, to demonstrate the accuracy of our proposed method.

In this paper, we aim to reproduce the DC3 results across various problems using the best-performing configurations.
Unlike the iterative solvers, DC3 returns its solution after a single neural network evaluation followed by a fixed number of correction steps, so it has no termination tolerance that can be tightened to reach a prescribed optimality gap.
We therefore report a DC3 solving time for a target only if the solutions returned by DC3 meet that target on all test instances; otherwise, the entry is marked ``Not met''.
DC3 does not fail to return solutions in these cases, so its achieved optimality gap, constraint violation, and other solution-quality metrics are still reported.

The MEL model and DC3 are each trained with 10000 examples, split into 8000 for training, 1000 for evaluation, and 1000 for testing.
Each MEL training example is a tuple consisting of a query point, its corresponding envelope value and gradient, whereas each DC3 example is a full problem instance.
In all examples, the MEL model is an ICNN with two softplus hidden layers, trained for 10000 epochs with AdamW at a learning rate of $5\times 10^{-3}$ and a gradient-loss weight of 5. We use 16 units per layer and a batch size of 16, except in Example~1, where they are 64 and 256, respectively. sMEL-ADMM uses $\gamma=1.6$ and $\rho=1$, except $\rho=10$ in Example~1.
DC3 uses two hidden layers of 512 ReLU units each with batch normalization, trained with Adam at a learning rate of $10^{-3}$ and a batch size of 64 for 5000 epochs.
Ground-truth optimal values for all experiments are obtained using high-accuracy settings of IPOPT and Mosek for nonlinear programming, and OSQP for quadratic programming.
The MEL network is trained in JAX and DC3 in PyTorch; the optimization baselines and (s)MEL-ADMM evaluations use Julia.
All optimization solvers and (s)MEL-ADMM inference run on the CPU of an Intel i9-285K with 128 GB RAM, and the neural network models are trained on an RTX 5090 GPU.

\subsection{Example 1: Convex quadratic programming}
\label{sec:qp}

We evaluate sMEL-ADMM on a constrained convex quadratic programming (QP) benchmark with $L_1$ regularization and compare it with OSQP, Clarabel, standard ADMM, and DC3.
The problem has $n=100$ decision variables, $n_{\mathrm{eq}}=50$ equality constraints, and $n_{\mathrm{ineq}}=50$ inequality constraints.
All methods are evaluated on the same 1000 test instances.
\begin{subequations}
\begin{align}
    \minimize & \quad \frac{1}{2}x^\top Q x + c^\top x + \lambda \|x\|_1 \\
    \subjectto & \quad A_{\mathrm{eq}} x = b_{\mathrm{eq}}, \\
    & \quad A_{\mathrm{ineq}} x \leq b_{\mathrm{ineq}},
\end{align}
\end{subequations}
where $Q \in \mathbb{R}^{n\times n}$ is a positive definite matrix, $c \in \mathbb{R}^n$, $A_{\mathrm{eq}} \in \mathbb{R}^{n_{\mathrm{eq}}\times n}$, $b_{\mathrm{eq}} \in \mathbb{R}^{n_{\mathrm{eq}}}$, $A_{\mathrm{ineq}} \in \mathbb{R}^{n_{\mathrm{ineq}}\times n}$, and $b_{\mathrm{ineq}} \in \mathbb{R}^{n_{\mathrm{ineq}}}$ are randomly generated.
We take matrices $A_{\mathrm{eq}},A_{\mathrm{ineq}}$ with entries drawn i.i.d. from the unit normal distribution.
The matrix $Q = B B^\top + \frac{1}{10} I, B \in \mathbb{R}^{100\times 100},\ B_{ij} \sim \frac{1}{20} \mathcal{N}(0,1)$ and $\lambda = 1$.
We assume that in each problem instance,
all entries of $b_{\mathrm{eq}}$ are in the interval $[-1,1]$.
In order to ensure that the problem is feasible, 
$b_{\mathrm{ineq}}$ is tied to $A_{\mathrm{ineq}}$ so that every instance is feasible. 
For any $b_{\mathrm{eq}} \in [-1, 1]^{50}$, the point $x_0 = A_{\mathrm{eq}}^+ b_{\mathrm{eq}}$ satisfies $A_{\mathrm{eq}} x_0 = b_{\mathrm{eq}}$ where $A_{\mathrm{eq}}^+$ is the Moore-Penrose pseudoinverse of $A_{\mathrm{eq}}$, and
$$(A_{\mathrm{ineq}} x_0)_i = (A_{\mathrm{ineq}} A_{\mathrm{eq}}^{+})_i b_{\mathrm{eq}} \le \sum_j \big|(A_{\mathrm{ineq}} A_{\mathrm{eq}}^{+})_{ij}\big| \cdot |b_{\mathrm{eq}}|_\infty \le b_{\mathrm{ineq},i} .$$

To solve the $L_1$-regularized benchmark with OSQP, we introduce auxiliary variables to obtain the following equivalent convex QP:
\begin{subequations}
\begin{align*}
    \minimize_{x,\, s} & \quad \frac{1}{2}x^\top Q x + c^\top x + \lambda \mathbf{1}^\top s \\
    \subjectto & \quad A_{\mathrm{eq}} x = b_{\mathrm{eq}}, \\
    & \quad A_{\mathrm{ineq}} x \leq b_{\mathrm{ineq}},\, \quad -s \leq x \leq s.
\end{align*}
\end{subequations}
The MEL model combines an ICNN with two hidden layers of 64 units each and softplus activations with a Huber term  $\sum_{i=1}^n C_i H_{\delta_i}(B_i^\top q + D_i)$, where
$H_\delta(t) = \begin{cases}\frac{1}{2\delta} t^2, & |t| \leq \delta \\ |t| - \frac{\delta}{2}, & |t|> \delta \end{cases}$ and
$C_i, B_i, D_i$ are trained parameters.
In this example, each instance is generated from a randomly sampled vector $b_{\mathrm{eq}}$.

\newcommand{\e}{\text{e}\!-\!}
\begin{table}[!tb]
\centering
\caption{QP benchmark results over 1000 shared test instances with $(n,n_{\mathrm{eq}},n_{\mathrm{ineq}})=(100,50,50)$. Entries report the mean (maximum). The optimality gap of DC3 exceeds all three targets, so no DC3 solving time is reported, while its best solution-quality metrics are reported.
\label{tab:qp_results}}
\begin{tabular}{!\vrule c|c|c|c|c|c|}
\cline{2-6}
\multicolumn{1}{c|}{} & $\xmm[delimiters = empty]{\text{OSQP}\\\text{mean} \\ \text{(max)}}$
& $\xmm[delimiters = empty]{\text{Clarabel}\\\text{mean} \\ \text{(max)}}$
& $\xmm[delimiters = empty]{\text{ADMM}\\\text{mean} \\ \text{(max)}}$
& $\xmm[delimiters = empty]{\text{sMEL-ADMM}\\\text{mean} \\ \text{(max)}}$
& $\xmm[delimiters = empty]{\text{DC3}\\\text{mean} \\ \text{(max)}}$
\\ \hline
$\xmm[delimiters = empty]{\text{Solving time (ms)}\\(g_{opt}\leq 1\%)}$
&$\xmm[delimiters = empty]{39.25; (467.27)}$ 
&$\xmm[delimiters = empty]{9.34; (210.79)}$ 
&$\xmm[delimiters = empty]{29.62; (49.54)}$ 
&$\xmm[delimiters = empty]{\textbf{2.83}; \textbf{(5.79)}}$ &Not met
\\ \hline
$\xmm[delimiters=empty]{\text{Solving time (ms)}\\(g_{opt}\leq 0.1\%)}$
&*
&*
& $\xmm[delimiters=empty]{82.97 ; (209.97)}$  
&$\xmm[delimiters=empty]{\textbf{7.62}; \textbf{(14.25)}}$
& Not met
\\ \hline
$\xmm[delimiters=empty]{\text{Solving time (ms)}\\(g_{opt}\leq 0.01\%)}$
&* 
&*
& $\xmm[delimiters=empty]{159.96 ; (805.85)}$  
& \textbf{$\xmm[delimiters=empty]{15.46 ; (56.63)}$} & Not met
\\ \hline
$\xmm[delimiters=empty]{\text{Sparsity}\\ \text{Zeros} (\%)}$ 
&$\xmm[delimiters=empty]{48.3; (52.0)}$ 
&$\xmm[delimiters=empty]{48.3; (52.0)}$	
&$\xmm[delimiters=empty]{48.3; (52.0)}$	
&$\xmm[delimiters=empty]{48.2; (52.0)}$	
&$\xmm[delimiters=empty]{0.4; (3.0)}$
\\
\hline
Constraint violation & - & $\xmm[delimiters=empty]{\num{9.1e-14} ; (\num{8.9e-13})}$ 
& $\xmm[delimiters=empty]{\num{5.8e-14} ; (\num{1.5e-14})}$ & $\xmm[delimiters=empty]{\num{1.9e-15} ; (\num{4.1e-15})}$ & $\xmm[delimiters=empty]{\num{2.1e-14} ; (\num{6.0e-14})}$
\\ \hline
Optimality gap (\%) & -  
&$\xmm[delimiters=empty]{\num{5.52e-6}; (\num{2.38e-5})}$ 
&$\xmm[delimiters=empty]{\num{9.5e-3} ; (\num{1e-02})}$ 
& $\xmm[delimiters=empty]{\num{9.7e-03} ; (\num{1e-02})}$ 
&$\xmm[delimiters=empty]{20.5 ; (58.3)}$
\\ \hline
\multicolumn{6}{l}{\footnotesize``Not met'' indicates that the solutions returned by DC3 do not meet the target.}\\
\multicolumn{6}{l}{\footnotesize``*'' denotes a result similar to that reported in the preceding row (target already met).}\\
\multicolumn{6}{l}{\footnotesize``–'' denotes omitted results, as OSQP serves as ground truth.}
\end{tabular}
\end{table}

Table~\ref{tab:qp_results} shows that sMEL-ADMM achieves the lowest reported mean and maximum solving times at the $1\%$ and $0.1\%$ targets, and the lowest maximum solving time at the $0.01\%$ target. At $g_{opt}\leq 1\%$, its mean solving time is 2.83 ms, corresponding to speedups of approximately $13.9\times$, $3.3\times$, and $10.5\times$ over OSQP, Clarabel, and standard ADMM, respectively. Its maximum solving time is 5.79 ms, compared with 467.27 ms for OSQP, 210.79 ms for Clarabel, and 49.54 ms for ADMM. 
Thus, the acceleration also benefits the slowest observed instances in this test set.

The advantage over standard ADMM persists as the accuracy target is tightened. At $g_{opt}\leq 0.1\%$ and $g_{opt}\leq 0.01\%$, sMEL-ADMM is approximately $10.9\times$ and $10.3\times$ faster in mean solving time, respectively. Both methods require more time at tighter targets, but sMEL-ADMM retains an order-of-magnitude reduction in mean solving time. 
OSQP and Clarabel are classical solvers, at the settings used for the $1\%$ target, they already return solutions with gaps far below $0.01\%$ (Clarabel's mean gap is $5.5\times 10^{-6}\%$, and OSQP provides the ground truth).
Their solving times at the tighter targets are therefore the same as those in the first row.
Against these times, sMEL-ADMM remains about $5.2\times$ and $1.2\times$ faster than OSQP and Clarabel in mean solving time at $g_{opt}\leq 0.1\%$.
At $g_{opt}\leq 0.01\%$, it is about $2.5\times$ faster than OSQP, but its mean solving time of 15.46 ms exceeds Clarabel's 9.34 ms, although its maximum solving time remains lower (56.63 ms versus 210.79 ms).

The solution-quality metrics complement the timing results. sMEL-ADMM returns a mean zero-entry percentage of 48.2\%, close to the 48.3\% reported for OSQP, Clarabel, and ADMM, while maintaining reported constraint violations on the order of $10^{-15}$. Its mean optimality gap of $9.7\times 10^{-3}\%$ is also comparable to ADMM's $9.5\times 10^{-3}\%$. In contrast, DC3 has small reported constraint violations but a mean optimality gap of 20.5\% and only 0.4\% zero entries on average. Since this gap exceeds all three targets, no DC3 solving time is reported in Table~\ref{tab:qp_results} (``Not met''). These results show that feasibility alone does not ensure objective accuracy or recovery of the sparse solutions induced by the $L_1$ penalty. The reported solving times do not establish an end-to-end speedup that includes model training.

\subsection{Example 2: Energy management in a microgrid}
\label{sec:power_system}


We consider the microgrid in \citet{risbeckEconomicModel2020,cortes-aguirreEconomicMPC2025}, which includes photovoltaic (PV) generation, controllable loads $p_t^c$, uncontrollable loads $p_t^u$, a battery energy storage system (BESS), and a utility-grid connection.
The control objective is to schedule the BESS and controllable loads to reduce electricity costs while maintaining user comfort.
We use IPOPT and MadNLP as solver baselines for this example.

At a prediction step $k$, the electricity cost includes the energy cost $J_{e,k}$ and the peak demand charge $J_{p,k}$.
The energy cost is
\(J_{e,k} = r_{ec} \Delta_T \left( m_{k|t} + \frac{1-\eta}{2\sqrt \eta} | u_{k|t}| \right)\),
where $r_{ec}$, $\Delta_T$, and $\eta$ are respectively the energy charge rate, time-step length, and round-trip efficiency factor of the BESS.
In addition, $m_{k|t}$ and $u_{k|t}$ denote the predicted imported power and the predicted charging/discharging power of the BESS at $k$ steps ahead of the current time $t$.
The BESS discharges when $u_{k|t} > 0$ and charges when $u_{k|t} < 0$, following
\begin{align}
\label{eq:bat_dyn}
x_{k+1|t} = x_{k|t}  - \frac{u_{k|t}}{b_{ess}} \Delta_T,
\end{align}
where $b_{ess} > 0$ is the BESS energy capacity (kWh) and $x_{k|t}$ is the predicted state of charge.
Let $r_{op}$ be the on-peak demand charge rate.
The peak demand charge $J_{p,k}$ is given by $J_{p,k} =  r_{op} \max \left( m_{k|t}, 0 \right)$.
\textit{Discomfort cost}: At $k$ steps ahead of the current time $t$, the discomfort cost associated with the controllable load $p_{k|t}^c$ is given by $J_{d,k}  =  r_{df}f(p_{k|t}^c)$,
where $f(p) =  +\infty$ if $p \leq 0$,
$f(p)=\frac{a}{p} - 1$ if $0 < p < a$, and $f(p) = 0$ if $p \geq a$.
The function $f$ measures customer discomfort, and $r_{df}$ is the discomfort cost weight.
Here, $a$ is the preferred power consumption of the controllable load.
If  $p_{k|t}^c$ is lower than $a$, the discomfort level is high, and discomfort decreases as the supplied power approaches the preferred level.

For an $N$-step horizon, the objective is to minimize the total electricity and discomfort cost, $J \!=\! \sum_{k=1}^{N} (J_{e,k} + J_{p,k} + J_{d,k})$.
Let $\hat g_{k|t}$ and $\hat p_{k|t}^u$ denote the forecasted renewable generation and forecasted uncontrollable demand $k$-step ahead.
The power balance is then given by
\begin{align}
p_{k|t}^c + \hat p_{k|t}^u = m_{k|t} + u_{k|t} + \hat g_{k|t}.
\label{eq:power_flow}
\end{align}
The renewable generation and demand are forecasted over the prediction horizon at each time step $t$.
The receding horizon planning problem is thus formulated as:
\begin{subequations}
\label{eq:MPC-ex2}
\begin{optimization}{J(\{p_{k|t}^c, u_{k|t}, m_{k|t}\}_{k=1,\dots,N})\label{eq:MPC-ex2-cost}}
&\eqref{eq:bat_dyn}, \eqref{eq:power_flow}, \\
& u_{\max} \geq u_{k|t} \geq u_{ \min}, \label{eq:MPC-ex2-bounded-cstr}\\
& x_{\max} \geq x_{k|t} \geq x_{\min},\;x_{N|t} \geq \underline{x}_N\text.
\label{eq:MPC-ex2-max-cstr}
\end{optimization}
\end{subequations}
Because our methods can handle the nonsmooth cost \eqref{eq:MPC-ex2-cost}, we use \eqref{eq:MPC-ex2} to collect data for training the MEL model described in Section~\ref{sec:MEL-ADMM}.
To apply IPOPT and MadNLP, however, we transform \eqref{eq:MPC-ex2} into a smooth problem by introducing auxiliary variables $s_{u,k}$, $s_{m,k}$ and $s_{d,k}$ then reformulating \eqref{eq:MPC-ex2} as
\begin{subequations}
\label{eq:MPC-relax}
\begin{optimization}{ \sum_{k=1}^{N} \! r_{ec} \Delta_T \!\left( m_{k|t} \!+\! \frac{1\!-\!\eta}{2\sqrt \eta} s_{u,k} \right)
\!+\! r_{op} s_{m,k} \!+\!  r_{df} s_{d,k}\label{eq:MPC-relax-cost}}
&\eqref{eq:bat_dyn},\eqref{eq:power_flow}, \eqref{eq:MPC-ex2-bounded-cstr} \text{~and~} \eqref{eq:MPC-ex2-max-cstr},
\\
&s_{u,k} \geq u_{k|t},\;s_{u,k} \geq - u_{k|t},\; s_{m,k} \geq m_{k|t}, \;s_{m,k} \geq 0,
\\
& s_{d,k} \geq \frac{a}{p_{k|t}^c} - 1,\;s_{d,k} \geq 0.
\label{eq:MPC-relax-cstr}
\end{optimization}
\end{subequations}
The discomfort function assigns an infinite cost when $p_{k|t}^c\leq 0$, so any finite-cost solution has $p_{k|t}^c>0$.
The reformulation \eqref{eq:MPC-relax} is therefore equivalent to \eqref{eq:MPC-ex2}, with the nonlinear objective terms represented by auxiliary variables and constraints.

\begin{table}[!tb]
\centering
\caption{Mean (maximum) solving time and optimality gap for the microgrid example at $N=96$ over 1000 samples.}
\begin{tabular}{!\vrule c|c|c|c|}
\cline{2-4}
\multicolumn{1}{c|}{} &\multicolumn{3}{c|}{$N=96$ (\# scalar variable = {288})}
\\
\cline{2-4}
\multicolumn{1}{c|}{} & $\begin{matrix}  \text{Solving time (ms)} \\ (g_{opt} \leq 1\%) \end{matrix}$
& $\begin{matrix}  \text{Solving time (ms)} \\ (g_{opt} \leq 0.01\%) \end{matrix}$
& $\begin{matrix}  \text{Lowest optimality gap (\%)} \\ \text{mean (max)} \end{matrix}$
\\
\hline
ADMM & 145.12 (229.38)   & 236.66 (347.53) & \num{5.2e-4} (\num{6.6e-4})
\\
\hline
IPOPT & 13.36 (43.37)  & 15.20 (47.41) & -
\\
\hline
MadNLP &9.57 (23.67)    & 10.77 (24.75) &  \num{2.5e-10} (\num{1.3e-9})
\\
\hline
DC3 &7.34 (13.70) &Not met &\num{0.15} (\num{0.62})
\\
\hline
MEL-ADMM &18.31 (27.05) &34.15 (49.23) & \num{4.3e-3} (\num{5.2e-3})
\\
\hline
sMEL-ADMM &\textbf{1.29 (2.41)} &\textbf{2.47 (3.69)} & \num{5.1e-3} (\num{6.5e-3})
\\
\hline
\multicolumn{4}{l}{\footnotesize ``–'' denotes omitted results, as IPOPT serves as ground truth.}\\
\multicolumn{4}{l}{\footnotesize``Not met'': the solutions returned by DC3 do not meet the target on all samples.}
\end{tabular}
\label{tab:power_runtime1}
\vspace{1em}
\caption{Mean (maximum) solving time and optimality gap for the microgrid example at $N=192$ over 1000 samples. \label{tab:power_runtime2}
}
\begin{tabular}{!\vrule c|c|c|c|}
\cline{2-4}
\multicolumn{1}{c|}{} &\multicolumn{3}{c|}{$N=192$ (\# scalar variable = {576})}
\\
\cline{2-4}
\multicolumn{1}{c|}{} & $\begin{matrix}  \text{Solving time (ms)} \\ (g_{opt} \leq 1\%) \end{matrix}$
& $\begin{matrix}  \text{Solving time (ms)} \\ (g_{opt} \leq 0.01\%) \end{matrix}$
& $\begin{matrix}  \text{Lowest optimality gap (\%)} \\ \text{mean (max)} \end{matrix}$
\\
\hline
ADMM  & 271.63 (342.84) &471.91 (668.60) &\num{5.4e-4} (\num{6.7e-4})
\\
\hline
IPOPT  & 25.50 (110.14) &27.23 (111.41)  & -
\\
\hline
MadNLP &18.83 (123.22)	 &20.28 (124.90) & \num{1.2e-9} (\num{5.6e-9})
\\
\hline
DC3 &Not met &Not met &\num{0.08} (\num{7.44})
\\
\hline
MEL-ADMM  &32.86 (41.57) &58.94 (80.28) &\num{4.8e-3} (\num{8.7e-3})
\\
\hline
sMEL-ADMM  &\textbf{1.38 (2.59)	} &\textbf{2.76 (4.12)} &\num{5.4e-3} (\num{9.2e-3})
\\
\hline
\multicolumn{4}{l}{\footnotesize ``–'' denotes omitted results, as IPOPT serves as ground truth.}\\
\multicolumn{4}{l}{\footnotesize``Not met'': the solutions returned by DC3 do not meet the target on some samples.}
\end{tabular}
\vspace{-1em}
\end{table}

For this setup, we use the following parameters (partly taken from \citet{cortes-aguirreEconomicMPC2025}):
$\Delta_T = \qty{0.25}{\hour}$,
$N \in \{96,192\}$ (corresponding to 24 and 48 hours, respectively),
$r_{ec} = \$0.1/\unit{\kWh}$,
$r_{op} = \$19.19/\unit{\kW}$,
$r_{df} = \$10$,
initial SOC $x_0 = 0.5$ (50\%),
$u_{\max} = \qty{700}{\kW}$,
$u_{\min} = \qty{-700}{\kW}$,
$b_{ess} =  \qty{500}{\kWh}$,
$\eta = 0.8$,
$x_{\min} = 0.2$ (20\%),
$x_{\max} = 0.8$ (80\%),
$\underline{x}_N = 0.5$ (50\%),
and $a = \qty{50}{\kW}$.
Perfect forecasts are assumed for power demand and PV generation.
Using IPOPT with $\epsilon_{tol} = 10^{-8}$, the optimal objective values for $N=96$ (one day) and $N = 192$ (two days) are $J_{N=96}^\star = 36479.1$ and $J_{N=192}^\star = 73117.5$, respectively.
In sMEL-ADMM, the per-step objective and penalty parameter are unchanged between $N=96$ and $N=192$, so the same MEL network can be reused for both horizons.

Tables~\ref{tab:power_runtime1} and~\ref{tab:power_runtime2} compare the mean (maximum) solving times over 1000 samples at two prediction horizons and two optimality-gap targets.
sMEL-ADMM achieves the lowest mean and maximum solving times in all four configurations.
At $N=96$ and $g_{opt}\leq 1\%$, its mean solving time is $1.29$~ms, compared with $9.57$~ms for MadNLP, $13.36$~ms for IPOPT, and $145.12$~ms for vanilla ADMM, corresponding to speedups of approximately $7.4$, $10.4$, and $112.5$, respectively.
DC3 also meets this gap target at $N=96$, but requires $7.34$~ms on average, approximately $5.7$ times the solving time of sMEL-ADMM.
At $g_{opt}\leq 0.01\%$, sMEL-ADMM requires $2.47$~ms at $N=96$ and $2.76$~ms at $N=192$, yielding mean speedups of $4.4$--$7.3$ over MadNLP and $6.2$--$9.9$ over IPOPT.
Doubling the horizon increases its mean solving time by only $7\%$ at the $1\%$ target and $12\%$ at the $0.01\%$ target.
Thus, its computational advantage persists at the longer horizon and tighter accuracy target.

\begin{figure}[!tb]
\centering
\subfloat[$g_{opt} \leq 1\%$]{\includegraphics[width = 0.5\linewidth]{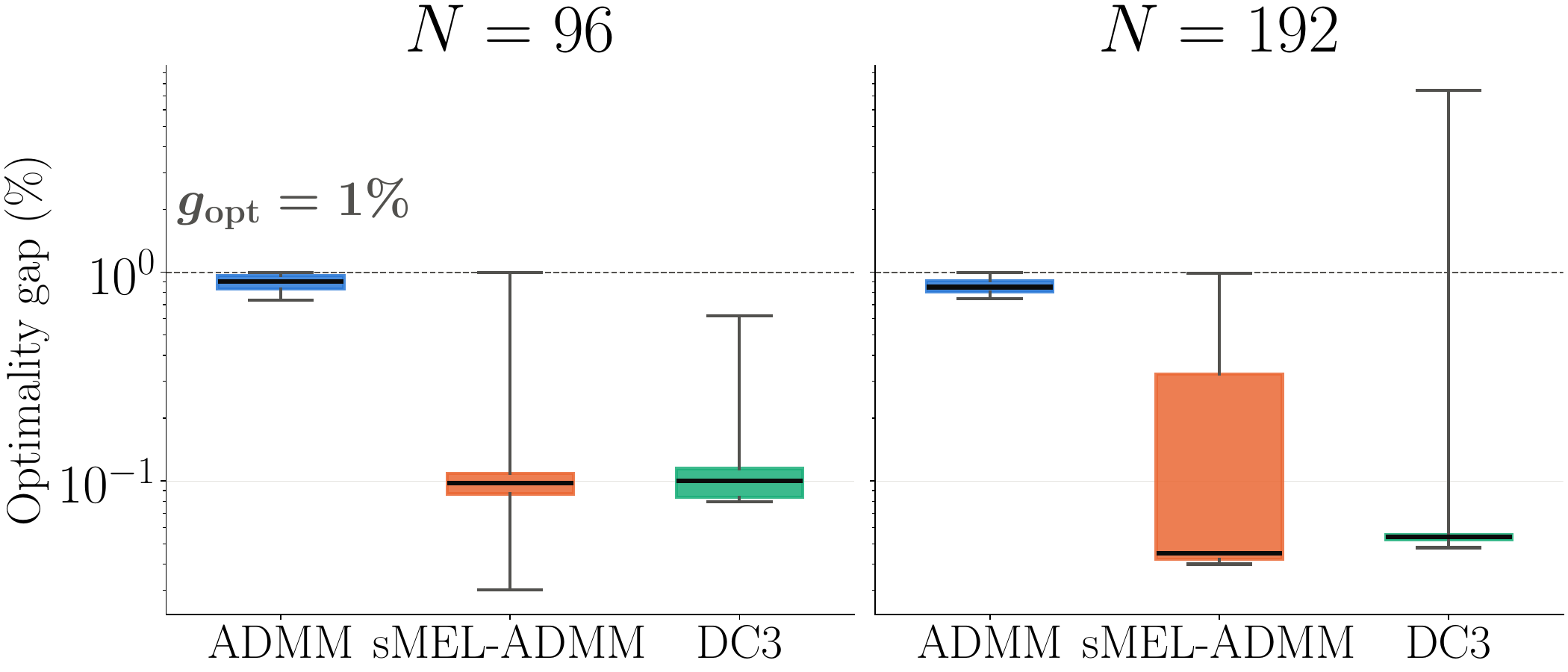}}
\subfloat[$g_{opt} \leq 0.01\%$]{\includegraphics[width = 0.5\linewidth]{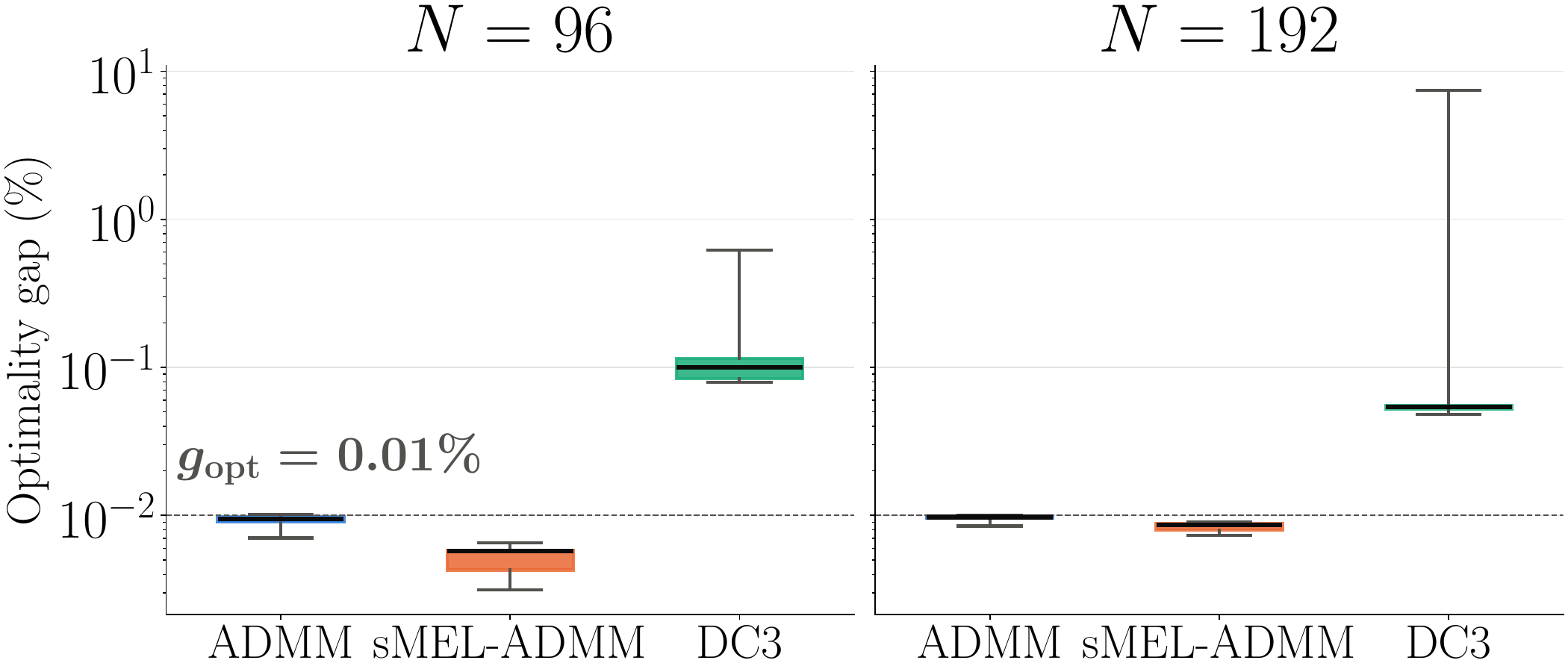}}
\caption{Optimality gaps of vanilla ADMM, sMEL-ADMM, and DC3 at $N=96$ and $N=192$, evaluated over 1000 samples per horizon. Panels (a) and (b) show results for targets of $1\%$ and $0.01\%$, respectively; dashed lines mark the corresponding targets.}
\label{fig:opt_gap_0.01_power_grid}
\end{figure}
\begin{figure}[!tb]
\centering
\subfloat[$g_{opt} \leq 1\%$]{\includegraphics[width = 0.5\linewidth]{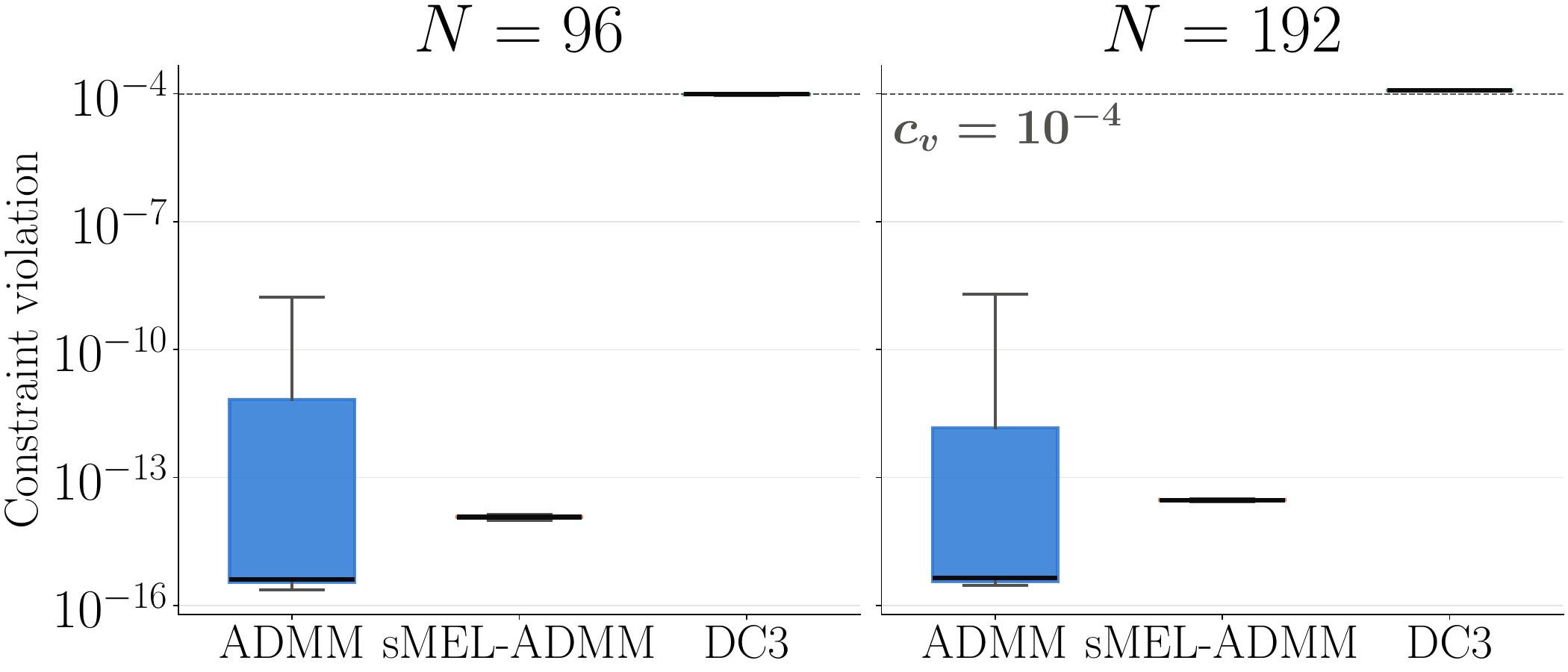}}
\subfloat[$g_{opt} \leq 0.01\%$]{\includegraphics[width = 0.5\linewidth]{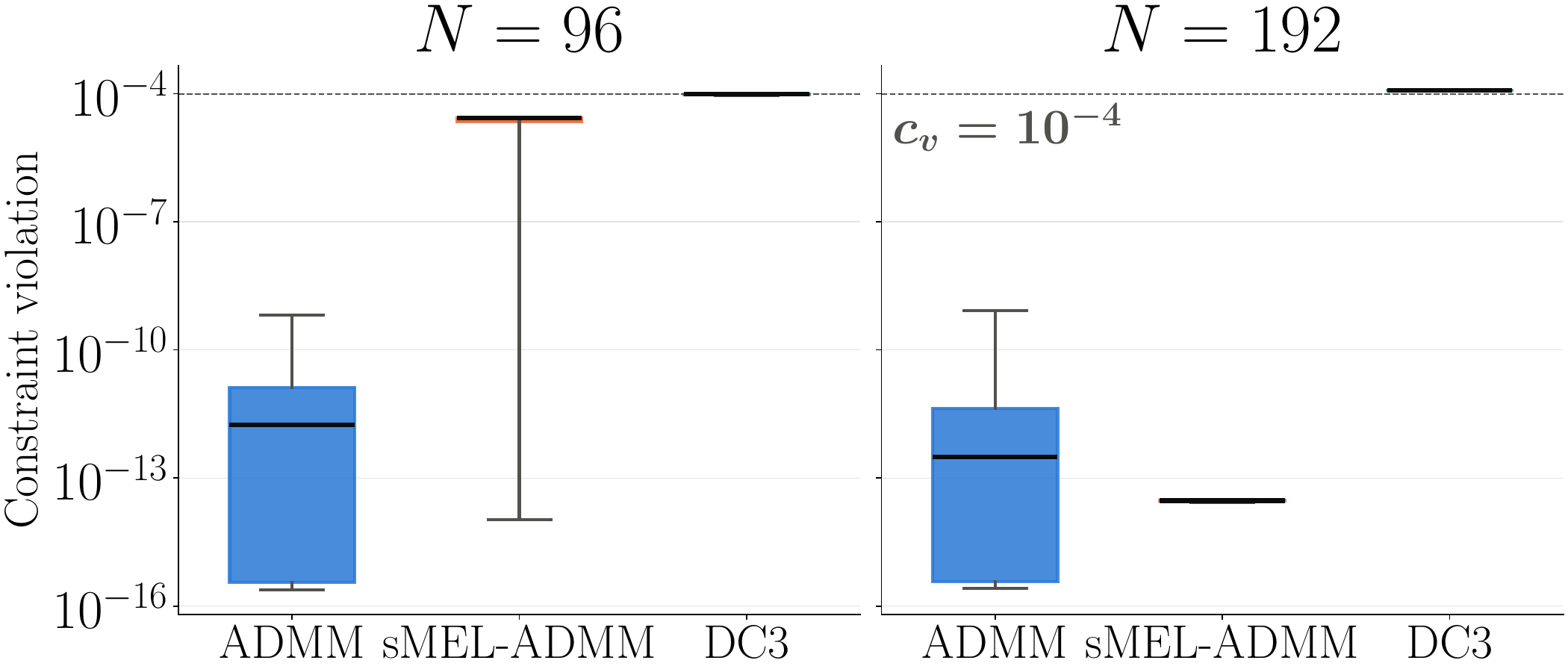}}
\caption{Constraint violations of vanilla ADMM, sMEL-ADMM, and DC3 at $N=96$ and $N=192$, evaluated over 1000 samples per horizon. Panels (a) and (b) correspond to optimality-gap targets of $1\%$ and $0.01\%$, respectively; dashed lines mark $c_v=10^{-4}$.}
\label{fig:cstr_viol_power_grid}
\vspace{-1em}
\end{figure}

Figure~\ref{fig:opt_gap_0.01_power_grid} shows that vanilla ADMM and sMEL-ADMM meet both optimality-gap targets at both horizons.
sMEL-ADMM achieves mean (maximum) gaps of $0.0051\%$ ($0.0065\%$) at $N=96$ and $0.0054\%$ ($0.0092\%$) at $N=192$.
DC3 meets the $1\%$ target only at $N=96$ and fails to meet the $0.01\%$ target at either horizon; its maximum gap rises from $0.62\%$ to $7.44\%$ when the horizon doubles.
Figure~\ref{fig:cstr_viol_power_grid} shows that vanilla ADMM and sMEL-ADMM maintain constraint violations below $c_v=10^{-4}$ at both targets and horizons.
DC3's violations lie near this limit at $N=96$ and slightly above it at $N=192$.
Overall, sMEL-ADMM provides the fastest solutions while meeting both accuracy targets and the feasibility threshold.

\subsection{Example 3: Entropy maximization}
\label{sec:max_entropy}

This example considers the constrained entropy maximization problem in \citet[Eq.~(5.13)]{boyd2004convex}, expressed as
\begin{subequations} \label{eq:maxEntropy}
\begin{optimization}[\maximize_{x > 0}]
{-\sum_{i=1}^{n} x_i \log x_i}
& {\bf 1}^\top x = 1,
\\
&Ax \leq b, \label{eq:Axb}
\end{optimization}
\end{subequations}
where $x = [x_1, x_2,\dots, x_n] \in \mathbb{R}^n$ is the decision variable representing a probability distribution, and
$A \in \mathbb{R}^{m \times n}$ and $b \in \mathbb{R}^m$ define the linear inequality constraints.
Here, $A$ and $b$ are treated as problem parameters.
To apply sMEL-ADMM (Algorithm~\ref{alg:ADMM-MEL-split}) to this problem, we reformulate \eqref{eq:Axb} as $Ax = s$ and $s \leq b$ with auxiliary variable $s \in \mathbb{R}^m$.

This example evaluates the constraint satisfaction, accuracy, and computational performance of sMEL-ADMM across randomized problem instances.
Reference optimal objective values are obtained from high-accuracy IPOPT solves with tolerance $\epsilon=10^{-9}$.
We compare sMEL-ADMM with IPOPT and DC3.
The IPOPT baseline is a separate run whose termination tolerance is tuned to meet the targets below, as described in Section~\ref{sec:experiments}.
Its optimality gaps, shown in Fig.~\ref{fig:opt_gap_entr_max}, are therefore nonzero and are measured against the high-accuracy reference solves.
Each entry of $A$ is sampled uniformly from $[0,1]$, and $b$ is generated so that \eqref{eq:Axb} is feasible.
We evaluate four problem sizes, $(n,m)=(100,1)$, $(100,10)$, $(1000,10)$, and $(1000,100)$, using randomly generated $(A,b)$ instances.
The target optimality gap is $g_{opt}\leq 0.1\%$.
{In addition, the MEL model in sMEL-ADMM is trained only once and applied to all four problem sizes, whereas the deep learning model in DC3 must be trained separately for each problem size.}

\begin{figure}[!tb]
\centering
\includegraphics[width = \linewidth]{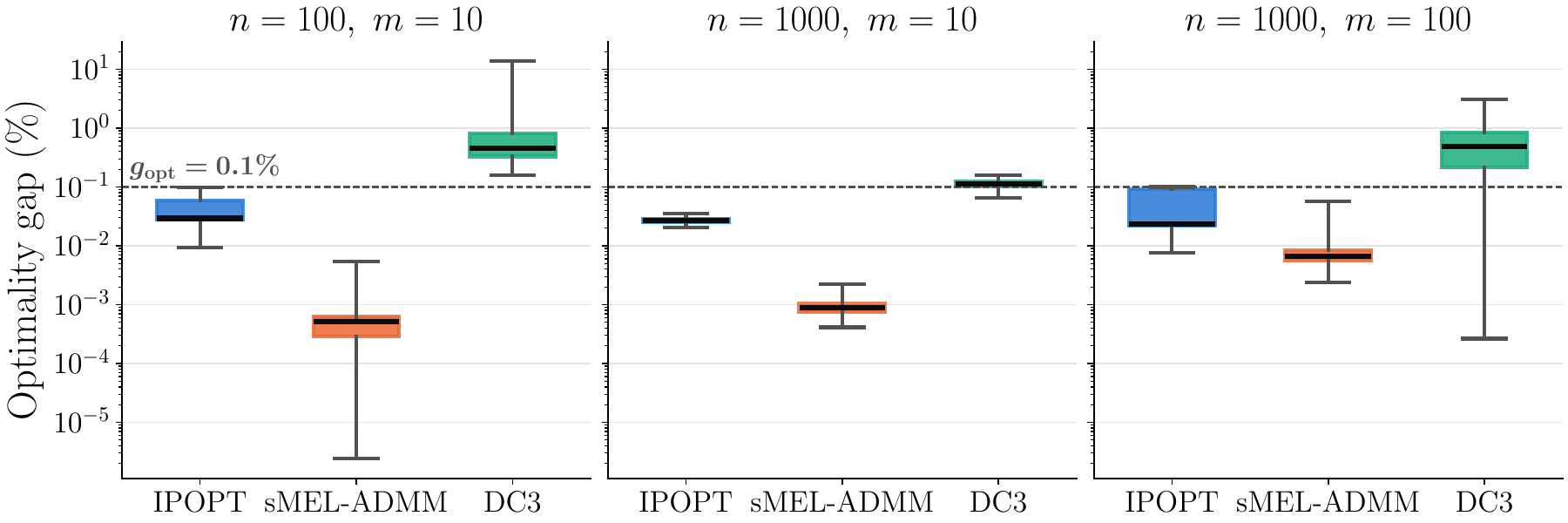}
\caption{Optimality gaps of IPOPT, sMEL-ADMM, and DC3 for three problem sizes, each with 1000 randomized instances.
The dashed line marks the 0.1\% target. IPOPT denotes the tuned-tolerance baseline; gaps are relative to the high-accuracy IPOPT reference.
\label{fig:opt_gap_entr_max}
}
\end{figure}

\begin{figure}[!tb]
\centering
\includegraphics[width = \linewidth]{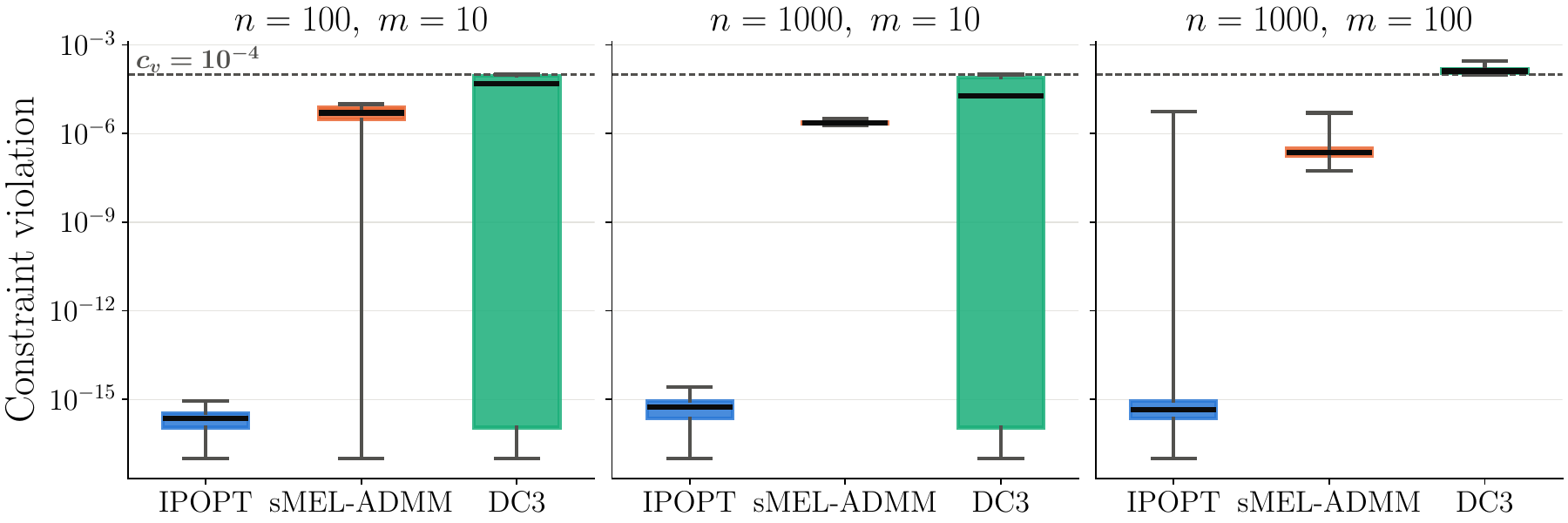}
\caption{Constraint violations of IPOPT, sMEL-ADMM, and DC3 for three problem sizes, each with 1000 randomized instances.
The dashed line marks the $10^{-4}$ limit.
\label{fig:cstr_viol_entr_max}
}
\vspace{-1em}
\end{figure}

Table~\ref{tab:IPOPT-DC3-LME} reports the mean and maximum time needed to meet the targets $g_{opt}\leq 0.1\%$ and $c_v\leq 10^{-4}$ over 1000 instances per problem size, together with the observed optimality gaps.
sMEL-ADMM has the lowest mean solving time at all four sizes. Its mean solving times are 0.52, 0.74, 1.41, and 11.87 ms for $(n,m)=(100,1)$, $(100,10)$, $(1000,10)$, and $(1000,100)$, respectively, compared with 0.53, 0.92, 5.61, and 50.27 ms for IPOPT. These correspond to mean speedups of approximately 1.0, 1.2, 4.0, and 4.2 times, respectively, so the advantage of sMEL-ADMM grows with the problem size.
At the smallest size, $(n,m)=(100,1)$, the two methods have nearly identical mean solving times, and IPOPT has a lower maximum solving time (0.79 ms versus 0.95 ms). At the three larger sizes, sMEL-ADMM also has the lowest maximum solving time.
The maximum optimality gap of sMEL-ADMM remains below the 0.1\% target at every size, ranging from $1.1\times 10^{-4}\%$ to $5.7\times 10^{-2}\%$. DC3 does not meet the optimality-gap target on all instances at any of the four sizes, so no DC3 solving time is reported. At $(n,m)=(100,1)$, its mean optimality gap is 0.04\%, but its maximum gap is 0.24\%, exceeding the target. At $(100,10)$, $(1000,10)$, and $(1000,100)$, its mean gaps are 0.84\%, 2.28\%, and 8.37\%, and its maximum gaps are 16.4\%, 16.9\%, and 12.5\%, respectively, all above the target.

Figures~\ref{fig:opt_gap_entr_max} and~\ref{fig:cstr_viol_entr_max} show the distributions of optimality gap and constraint violation, respectively, for the three sizes with $m\geq 10$.
At $(n,m)=(100,10)$ and $(1000,10)$, DC3's constraint violations approach the $10^{-4}$ limit.
At $(n,m) = (1000,100)$, DC3 exceeds the optimality-gap and constraint-violation limits, whereas sMEL-ADMM stays below both limits and achieves lower optimality gaps than IPOPT in the plotted distributions.

\begin{table}[!tb]
\centering
\caption{Mean (maximum) solving time to reach $g_{opt}\leq 0.1\%$ and $c_v\leq 10^{-4}$, and mean (maximum) optimality gap for entropy maximization over 1000 randomized instances per problem size.
IPOPT supplies the reference objective value.
DC3 does not meet the optimality-gap target on all instances at any problem size, so no DC3 solving time is reported, and its achieved optimality gaps are reported instead.
\label{tab:IPOPT-DC3-LME}
}
\begin{tabular}{!\vrule c|c|c|c|c|c|}
\cline{2-6}
\multicolumn{1}{c|}{} &$n$ &$m$ & $\begin{matrix} \text{IPOPT} \\ \text{mean (max)} \end{matrix}$  &$\begin{matrix} \text{sMEL-ADMM} \\ \text{mean (max)} \end{matrix}$ &$\begin{matrix} \text{DC3} \\ \text{mean (max)} \end{matrix}$
\\ \hline
\multirow{4}{*}{
$\begin{matrix}
\text{Solving time (ms)}
\\
(g_{opt} \leq 0.1\%)
\end{matrix}$
} &\multirow{2}{*}{100}  & 1    &0.53 (0.79) &\textbf{0.52} (0.95) &Not met
\\
& & 10 &0.92 (1.37) &\textbf{0.74 (1.04)}& Not met
\\ \cline{2-6}
&\multirow{2}{*}{1000} & 10   & 5.61 (5.98) &\textbf{1.41 (1.82)} &Not met
\\
& &100 &50.27 (66.12) &\textbf{11.87 (16.06)} &Not met
\\
\hline
\multirow{4}{*}{Optimality gap (\%)}
&\multirow{2}{*}{100} &1 & - &\num{5.1e-5} (\num{1.1e-4}) & 0.04 (0.24)
\\
& &10 & - &\num{4.8e-4} (\num{5.4e-3}) & 0.84 (16.4)
\\
\cline{2-6}
&\multirow{2}{*}{1000} &10 & - & \num{9.1e-4} (\num{2.3e-3}) & 2.28 (16.9)
\\
& &100 & - & \num{8.1e-3} (\num{5.7e-2}) & 8.37 (12.5)
\\\hline
\multicolumn{6}{l}{\footnotesize``Not met'': the solutions returned by DC3 do not meet the targets on some instances.}\\
\multicolumn{6}{l}{\footnotesize ``–'' denotes omitted results; see Fig.~\ref{fig:opt_gap_entr_max} for the IPOPT gaps.}
\end{tabular}
\end{table}

\section{Conclusion}
\label{sec:conclusion}

This paper presented LEAF, a learning-enabled ADMM framework for accelerated convex optimization through learned approximations of the Moreau envelope (ME).
By integrating an ICNN-based MEL model into ADMM iterations, the proposed framework preserves key structural properties of convex optimization, including convexity, smoothness, and convergence and feasibility guarantees.
Unlike existing learning-to-optimize approaches that directly approximate high-dimensional operators, LEAF learns a scalar-valued ME function, substantially reducing model complexity while improving scalability and data efficiency.
Based on LEAF, we developed the MEL-ADMM and sMEL-ADMM algorithms with established theoretical guarantees.
Numerical experiments on large-scale optimization problems demonstrated that the proposed methods achieve significant computational speedups over state-of-the-art solvers while maintaining low optimality gaps and small constraint violations, and that they achieve substantially lower optimality gaps than DC3, a representative learning-to-optimize method.
The results further show that exploiting optimization structure within the learning architecture can substantially improve both efficiency and reliability in learning-enabled optimization.
Future work will investigate extensions to broader problem classes, including mixed-integer and distributed convex optimization.
We also plan to explore applications in robotics, autonomous systems, and real-time control, where fast and reliable optimization is critical.

\acks{This material is based upon work supported by the National Science Foundation under Grants No.~2449927 and No.~2514584.
}

\appendix

\section{Proof of Lemma \ref{lem:LipschitzNN}
\label{appx:lipschitz}}

We will show that $\nabla_x^2 h(x)$ is uniformly bounded.
We first derive the expressions for $\nabla_x h(x)$ and $\nabla_x^2 h(x)$.
Let
\(D_j = \operatorname{diag}\big(\phi_j'(y_j)\big)\) and
\(
E_j = \operatorname{diag}\big(\phi_j''(y_j)\big)
\).
The Jacobian $J_j = \frac{\partial v_j}{\partial x}$ is given by the chain rule as follows
\begin{align*}
J_0 &= D_0 W_0,\quad
J_j = D_j\big(W_j J_{j-1} + V_j\big), \; j=1,\dots,\ell-1,
\\
G_\ell &= \nabla_x y_\ell = W_\ell J_{\ell-1} + V_\ell,
\\
\nabla_x h &= D_\ell G_\ell^{\top} = \phi_\ell'(y_\ell) G_\ell^\top.
\end{align*}
Using the chain rule again, the second derivative for the output layer is derived as follows
\begin{align*}
\nabla_x^2 h
&=
\phi_\ell''(y_\ell) \nabla_x y_\ell  \nabla_x y_\ell^\top
+
\phi_\ell'(y_\ell) \nabla_x^2 y_\ell \nonumber
=  \phi_\ell''(y_\ell)\, G_\ell^{\top} G_\ell
\!+\!
\phi_\ell'(y_\ell)
\sum_{k=1}^{n_{\ell-1}} (W_\ell)_k \, \nabla_x^2 (v_{\ell-1})_k,
\end{align*}
where $(\cdot)_k$ denotes the $k$-th row (element) of a matrix (vector), and $n_j$ is the number of neurons at the $j$-th layer.
For the $i$-th neuron in layer $j$, we define $g_{j,i}^{\top} = \big(W_j J_{j-1} + V_j\big)_i$, then
\[
\nabla_x^2 (v_j)_i
=
\phi_j''\big((y_j)_i\big)\, g_{j,i} g_{j,i}^{\top}
+
\phi_j'\big((y_j)_i\big)\, \nabla_x^2 (y_j)_i ,
\]
with
$\nabla_x^2 (y_j)_i
=
\sum_{k=1}^{n_{j-1}} ((W_j)_i)_k \, \nabla_x^2 (v_{j-1})_k$.

In the input layer $j=0$, we have
$g_{0,i}^{\top} = (W_0)_i,~
\nabla_x^2 (v_0)_i
= \phi_0''\big((y_0)_i\big)\,
(W_0)_i^{\top}(W_0)_i$.
Because all weights are finite and $\phi_\ell'$ and $\phi_\ell''$ are bounded, each term in the summation remains bounded.
For instance,
$\Vert \nabla_x^2 (v_0)_i \Vert \leq H_0$ implies that
\begin{align*}
\Vert \nabla_x^2 (v_1)_i \Vert \leq &\Vert \phi_1''((y_1)_i) \Vert \Vert g_{1,i} g_{1,i}^\top \Vert
\!+\! \Vert \phi_1'((y_1)_i)\Vert H_0  \Vert W_1 \Vert_\infty.
\end{align*}
Finally, by applying the mean-value theorem, it follows that
\(
\|\nabla_x h(x_1) -\nabla_x h(x_2)\|_2
\le \sup_x \|\nabla_x^2 h(x)\|\,\|x_1-x_2\|_2,
\)
which proves that $\nabla_x h$ is globally Lipschitz-continuous.

\section{Proof of Lemma \ref{thm:convergence1}
\label{appx}}

We first visit some useful definitions and results.
\begin{definition}[Convex conjugate]
The convex conjugate of a function $f$ is defined as $f^\star(y) = \sup_{x \in \bbR^n} \{ x^\top y - f(x)\}$.
\end{definition}

\begin{lemma}[Fenchel biconjugate]
\label{lm:biconjugate}
\citep{rockafellar1998variational}
The Fenchel biconjugate of $f$ is the convex conjugate of $f^\star$, i.e., $f^{\star\star}(x) = \sup_{y \in \bbR^n} \{ y^\top x - f^\star(y)\}$.
If $f \in \Gamma_0(\bbR^n)$ then $f^{\star \star} = f$.
\end{lemma}

\begin{lemma}
\label{lm:dual_convex}
\citep[Proposition~12.60]{rockafellar1998variational}
For a convex differentiable $f$, $\nabla f$ is Lipschitz continuous with constant $\frac{1}{\sigma}$ if and only if $f^\star$ is $\sigma$-strongly convex.
\end{lemma}

\begin{lemma}[Infimal convolution] \label{lem:inf}
\citep[Exercise~3.12]{borwein2006convex}
For $f, g \in \Gamma_0(\bbR^n)$, we define the \textit{infimal convolution}
$(f\odot g)(y) = \inf_x \{ f(x) + g(y-x)\}$.
Then, the following statements hold
\begin{itemize}
\item[(i)] $f\odot g \in \Gamma_0(\bbR^n)$,
\item[(ii)] $(f\odot g)^\star =  f^\star + g^\star$.
\end{itemize}
\end{lemma}

Now, we are ready to prove Lemma \ref{thm:convergence1}.
We prove \textit{existence} and \textit{uniqueness} separately.

\textit{Existence:} According to Lemma~\ref{lm:dual_convex} (convex differentiable functions from $\bbR^n \mapsto \bbR$ belong to $\Gamma_0(\bbR^n)$), $g^\star$ is a $\lambda$-strongly convex function.
Let us define a function $f = (g^\star - \frac{\lambda}{2}\Vert .\Vert_2^2)^\star$.
By Definition \ref{def:strong_convexity}, $g^\star - \frac{\lambda}{2}\Vert .\Vert_2^2$ is a convex function.
By Lemma~\ref{lm:biconjugate},
$f^\star = (g^\star - \frac{\lambda}{2} \Vert . \Vert_2^2)^{\star \star} = g^\star - \frac{\lambda}{2} \Vert . \Vert_2^2$.
Therefore, $f \in \Gamma_0(\bbR^n)$.
Based on Lemma \ref{lem:inf}, the ME is a special case of infimal convolution with the function $q_\lambda(x)=\frac{1}{2\lambda}\Vert x\Vert_2^2$.
Since $q_\lambda^\star(y)=\frac{\lambda}{2}\Vert y\Vert_2^2$, we have
\[
(M_{\lambda}f)^\star = f^\star + q_\lambda^\star = f^\star + \frac{\lambda}{2} \Vert . \Vert_2^2 = g^\star.
\]
Hence, $M_{\lambda}f = g$.

\textit{Uniqueness:} Assume there exists another $\hat f \in \Gamma_0(\bbR^n)$ such that $M_\lambda \hat f=g$.
Using the same conjugate identity,
\[
g^\star=(M_\lambda \hat f)^\star=\hat f^\star+\frac{\lambda}{2}\Vert . \Vert_2^2.
\]
Thus $\hat f^\star=g^\star\!-\!\frac{\lambda}{2}\Vert . \Vert_2^2=f^\star$.
By Lemma~\ref{lm:biconjugate}, $\hat f \!=\!\hat f^{\star\star}\!=\!f^{\star\star}\!=\! f$.

\section{Proof of Theorem \ref{thm:cvg_MEL-ADMM}}
\label{appx:cvg-MEL-ADMM}

Since $M_{\frac{1}{\rho}}\hat f=g$, it follows from \eqref{eq:MErelationship} that
\eqref{eq:ADMM_gen_xi-z} is equivalent to
\begin{align*}
\hat z^{i+1} &= \prox_{\frac{1}{\rho} \hat{f}} (q^i) = \argmin_z \hat f(z)  + \frac{\rho}{2} \Vert w^i - z +  \bar \alpha^i\Vert_2^2.
\end{align*}
Therefore, the iteration of MEL-ADMM is equivalent to the scaled form of ADMM \eqref{eq:ADMM-gen-z}-\eqref{eq:ADMM-gen-b} for solving \eqref{eq:equivalent}.
Since $\bar{\calZ}$ is nonempty, closed, and convex, $\hat f\in\Gamma_0(\bbR^n)$, a saddle point exists by assumption~\eqref{eq:surrogate-saddle-point}, and $\gamma\in(0,2)$, the convergence theory of relaxed two-block ADMM applies \citep{eckstein1992douglas,boydDistributedOptimization2011}.
Because both constraint matrices of the consensus formulation are $I$ and $-I$, there exists a primal-dual solution $(\hat z^\star,\hat\alpha^\star)$ of \eqref{eq:equivalent} satisfying \eqref{eq:surrogate-saddle-point} such that
\[
\hat z^i\to\hat z^\star,\qquad
w^i\to\hat z^\star,\qquad
\bar\alpha^i\to\rho^{-1}\hat\alpha^\star,
\]
where $\hat z^\star\in\argmin_{z\in\bar{\calZ}}\hat f(z)$ and $\hat\alpha^\star$ is the unscaled Lagrange multiplier for the consensus constraint $w-z=0$.
\hfill$\square$

\section{Proof of Theorem \ref{thm:bounded-optimality-error}}
\label{appx:bounded-optimality-error}

From the convergence of MEL-ADMM, we have $\hat z^i\to\hat z^\star$ and $\bar\alpha^i\to \rho^{-1} \hat\alpha^\star$.
Taking limits in the learned update and using consensus gives
$\nabla g(\hat q^\star)=\hat\alpha^\star$.
Let $p^\star=\prox_{\frac{1}{\rho}f}(\hat q^\star)$.
Relation \eqref{eq:MErelationship} applied to $F$ and to $g$ gives $p^\star=\hat q^\star-\rho^{-1}\nabla F(\hat q^\star)$ and $\hat z^\star=\hat q^\star-\rho^{-1}\nabla g(\hat q^\star)$, so \eqref{eq:uniform-ME-gradient-error} implies
\[
\|p^\star-\hat z^\star\|_2\leq\frac{\varepsilon_g}{\rho}.
\]
The exact proximal optimality condition gives
$\nabla f(p^\star)=\rho(\hat q^\star-p^\star)=\nabla F(\hat q^\star)$.
Therefore, the vector $r=\nabla f(\hat z^\star)-\hat\alpha^\star$ satisfies
\begin{align*}
\|r\|_2
&\leq\|\nabla f(\hat z^\star)-\nabla f(p^\star)\|_2
+\|\nabla F(\hat q^\star)-\nabla g(\hat q^\star)\|_2 
\leq \left(1+\frac{L_f}{\rho}\right)\varepsilon_g=\eta.
\end{align*}
Since $-\hat\alpha^\star\in N_{\bar{\calZ}}(\hat z^\star)$,
$r\in\nabla f(\hat z^\star)+N_{\bar{\calZ}}(\hat z^\star)$, proving \eqref{eq:bounded-stationarity-error}.
Since $z^\star\in\bar{\calZ}$ and $-\hat\alpha^\star\in N_{\bar{\calZ}}(\hat z^\star)$, the normal-cone inequality gives $\hat\alpha^{\star\top}(\hat z^\star-z^\star)\leq 0$, and hence
$\nabla f(\hat z^\star)^\top(\hat z^\star-z^\star)
=r^\top(\hat z^\star-z^\star)+\hat\alpha^{\star\top}(\hat z^\star-z^\star)
\leq r^\top(\hat z^\star-z^\star)$.
Convexity of $f$ and this bound yield
\begin{align*}
f(\hat z^\star)-f(z^\star)
&\leq\nabla f(\hat z^\star)^\top(\hat z^\star-z^\star) \leq r^\top(\hat z^\star-z^\star)
\leq\eta\|\hat z^\star-z^\star\|_2.
\end{align*}

\section{Proof of Proposition \ref{prop:strongly-convex-error}}
\label{appx:strongly-convex-error}

Let $r=\nabla f(\hat z^\star)-\hat\alpha^\star$ as in Appendix~\ref{appx:bounded-optimality-error}, so that $\|r\|_2\leq\eta$.
Optimality of $z^\star$ implies
$0\in\nabla f(z^\star)+N_{\bar{\calZ}}(z^\star)$.
Strong monotonicity of $\nabla f$ and monotonicity of the normal cone give
\[
\mu\|\hat z^\star-z^\star\|_2^2
\leq r^\top(\hat z^\star-z^\star)
\leq\eta\|\hat z^\star-z^\star\|_2,
\]
which proves \eqref{eq:strongly-convex-solution-error}.
Applying the strong-convexity inequality at $\hat z^\star$ and using the inequality $\nabla f(\hat z^\star)^\top(\hat z^\star-z^\star)\leq r^\top(\hat z^\star-z^\star)$ from Appendix~\ref{appx:bounded-optimality-error} gives, with $d=\|\hat z^\star-z^\star\|_2$,
\[
f(\hat z^\star)-f(z^\star)
\leq r^\top(\hat z^\star-z^\star)-\frac{\mu}{2}d^2
\leq\eta d-\frac{\mu}{2}d^2
\leq\frac{\eta^2}{2\mu}.
\]

\bibliography{references}

\end{document}